%% file: main.tex
\documentclass[conference,compsoc]{IEEEtran}

\pdfoutput=1

\usepackage{hyperref}

\usepackage{amsmath}
\usepackage{amssymb}
\usepackage{mathtools}
\usepackage{amsthm}

\usepackage{enumitem}
\usepackage[capitalize,noabbrev]{cleveref}

\theoremstyle{plain}
\newtheorem{theorem}{Theorem}[section]

\newtheorem{lemma}[theorem]{Lemma}

\theoremstyle{definition}
\newtheorem{definition}[theorem]{Definition}
\newtheorem{assumption}[theorem]{Assumption}
\theoremstyle{remark}

\usepackage[textsize=tiny]{todonotes}

\newcommand{\sys}{SureFED}
\newcommand{\sysbase}{BayP2PFL}

\newcommand\numberthis{\addtocounter{equation}{1}\tag{\theequation}}

\usepackage{tikz,pgfplots}
\usepackage{algorithm}
\usepackage{algpseudocode}
\usepackage{algcompatible}

\usepackage{float}
\usepackage{caption}
\usepackage{color}
\usepackage{xcolor}
\usepackage{enumerate}

\usepackage{multirow}

\usepackage{subcaption}

\usepackage{comment}

\usepackage{diagbox}

\begin{document}
	
	\title{\sys{}: Robust  Federated Learning via Uncertainty-Aware  Inward and Outward Inspection}

	

	\DeclareRobustCommand{\IEEEauthorrefmark}[1]{\smash{\textsuperscript{\footnotesize #1}}}

	\author{\IEEEauthorblockN{Nasimeh Heydaribeni*\IEEEauthorrefmark{1}, Ruisi Zhang*\IEEEauthorrefmark{1},  Tara Javidi\IEEEauthorrefmark{1},
			Cristina Nita-Rotaru\IEEEauthorrefmark{2}, Farinaz Koushanfar\IEEEauthorrefmark{1} }
		\IEEEauthorblockA{\IEEEauthorrefmark{1} Department of Electrical and Computer Engineering, University of California, San Diego, 
			La Jolla, CA 92093-0021}
		\IEEEauthorblockA{\IEEEauthorrefmark{2} Khoury College of Computer Sciences, Northeastern University, 
			Boston, MA 02115 }
		\IEEEauthorblockA{* Equal contribution}}

	\maketitle
	\begin{abstract} 
		In this work, we introduce \sys, a novel framework for byzantine robust  federated learning. Unlike many existing defense methods that rely on statistically robust quantities, making them vulnerable to stealthy and colluding attacks,  \sys{} establishes trust using the local information of benign clients. \sys{} utilizes an uncertainty aware model evaluation and introspection to safeguard against poisoning attacks.  In particular, each  client independently trains a clean local model exclusively  using its local dataset, acting as the reference point for evaluating model updates.
		\sys{} leverages Bayesian models that provide model uncertainties and play a crucial role in the model evaluation process.
		Our framework exhibits robustness even when the majority of clients are compromised, remains agnostic to the number of malicious clients, and is well-suited for non-IID settings.
		We theoretically prove the robustness of our algorithm against data and model poisoning attacks in a decentralized linear regression setting. Proof-of-Concept evaluations on benchmark image classification data demonstrate the superiority of \sys{}  over the state of the art defense methods  under various colluding and non-colluding data and model poisoning attacks. 
	\end{abstract}

	\input{1_Intro_ver1}
	\input{preliminaries.tex}
	\input{2_Problem_Statement.tex}
	\input{SABRE.tex}

	\input{4_Learning_Robustness.tex}
\input{5_Experiments.tex}
	\input{7_Discussion}

	\newpage

\input{main.bbl}
	\onecolumn
	\input{Appendix.tex}

\end{document}

%% file: 1_Intro_ver1.tex
\section{Introduction}

In contemporary large networks, decentralized operations involve individual devices (clients) having access to their local datasets. However, these local datasets often lack the capacity to learn a global model with sufficient accuracy, and privacy concerns impede raw data sharing among clients. To address these challenges, federated learning algorithms \cite{mcmahan2017communication} like Federated Averaging have been proposed. Furthermore, an array of extensions were introduced to enable improved scalability in a peer-to-peer setting or adaptation to non-IID data distributions \cite{roy2019braintorrent,lalitha2019peer, wang2021non, wink2021approach, wang2022peer, al2020federated}.
However, federated learning algorithms have been shown to be vulnerable to various byzantine and adversarial attacks, including data and model poisoning attacks  \cite{lyu2020threats,sun2019can,tolpegin2020data,jere2020taxonomy,bhagoji2019analyzing,Yar2023Backdoor,severi2022network}. 
State of the art  defense methods in federated learning can be categorized into (1) statistics-based defenses   \cite{yin2018byzantine,  blanchard2017machine, boussetta2021aksel, chen2017distributed, guerraoui2018hidden}, such as Trimmed Mean \cite{yin2018byzantine} and Clipping \cite{bagdasaryan2020backdoor}, 
  (2) defenses based on model evaluation on a local dataset \cite{xie2019zeno, xie2020zeno++}, such as Zeno \cite{xie2019zeno}, and (3) defenses based on model comparison with the server's local model, such as  FLTrust \cite{cao2020fltrust}.
  
Statistics-based defense methods use a statistically robust quantity, such as median,  instead of mean, in the model aggregation to eliminate the effect of byzantine updates on the aggregated value. Such methods rely on the assumption that the majority of clients are benign, aiming to protect statistical measures like median from being influenced by a minority of poisoned updates. 
The defense methods in the second  category, such as Zeno \cite{xie2019zeno}, draw samples from a clean local dataset and evaluate a model update based the amount of improvement in the loss function. The defenses in the third category, such as FLTrust \cite{cao2020fltrust}, are based on bootstrapping the trust using the server's local model. FLTrust assumes that the server has access to a  clean local dataset. The server computes a model update using its local dataset  similar to  what clients do and compares received updates with its own. 

The existing defense methods have been shown to be vulnerable to specific attacks such as colluding attacks of  "A Little is Enough" \cite{baruch2019little} and Trojan attacks \cite{xie2020dba}. In addition, the  statistics-based defense methods require the majority of clients to be benign. Some methods, such as Trimmed Mean \cite{yin2018byzantine} and Zeno \cite{xie2019zeno}, necessitate knowledge of an upper bound on the number of compromised clients. Zeno is also inherently vulnerable to Trojan attacks due to their stealthy nature. FLTrust has also been shown to be vulnerable to specific attacks that target methods based on cosine similarity \cite{kasyap2022hidden}. A potential drawback of FLTrust is that if the server's model becomes poisoned early on in the training  process, the    future model evaluations become meaningless.

This paper introduces \sys{}, which utilizes uncertainty quantification  to address limitations in the existing literature. \sys{} draws inspirations from the state of the art defense methods described earlier and further incorporates additional essential components  to address the known vulnerabilities of the existing methods. \sys{} employs a  peer-to-peer federated learning setting in which each client has access to a local dataset but also obtains models from other clients. 
In contrast with prior work, in \sys{}, the clients train two distinct models:  a social model, designed to capture the benefits of federated operation by aggregating the models trained on diverse datasets across the network, and a \textbf{Local Bayesian Model}, trained exclusively  on each client's local dataset to be utilized as a model evaluation basis with explicit uncertainty quantification.
We note that these local models are never aggregated with the received  updates; 
Instead, they are used in an \textbf{Uncertainty
Aware Model Evaluation} of the model updates to
subsequently aggregate them with the social model based
on a new secure aggregation rule. Additionally, \sys{} employs an
introspection procedure, in which each client evaluates its own social model using its clean local model. This ensures the system’s ability to revert to a clean state after potential contamination.



The process of model evaluation and aggregation in \sys{} draws inspiration from the concepts found in social science such as bounded confidence models of opinion dynamics \cite{hegselmann2002opinion}, which aim to simulate how humans  alter their opinions during interactions with others. In this regard, \sys{} seeks to emulate the innate intelligence of humans within the complex domain of AI.
Notably, in \sys{}, each received model update is independently evaluated without the need for comparisons with other updates. This feature allows \sys{} to function without requiring a majority of benign clients or knowledge of the number of compromised clients. 

We investigate and demonstrate the superior  performance of \sys{} under a broad set of byzantine and adversarial attacks. In particular, we consider six colluding and non-colluding  data and model poisoning attacks, namely, Label-Flipping \cite{tolpegin2020data}, Trojan \cite{xie2020dba}, Bit-Flip \cite{xie2018generalized}, General Random \cite{xie2018generalized},  A little is Enough  \cite{baruch2019little}, and Gaussian \cite{pillutla2022robust} Attacks. A little is Enough and Trojan attacks are colluding attacks, while Trojan attack is also a stealthy attack.  
We compare our results 
with the state of the art defense methods 
described above. Specifically, we  compare \sys{} with 
Trimmed Mean \cite{yin2018byzantine}, and Clipping \cite{bagdasaryan2020backdoor} from the first category, 
  Zeno \cite{xie2019zeno}, which is the state of the art in the second category, and FLTrust~\cite{cao2020fltrust}, which is the state of the art defense belonging to the third category. We also compare with \sysbase{} \cite{wang2022peer}, as a variational peer-to-peer federated learning  baseline without a defense method to confirm that \sys{}'s robustness does not come at the cost of reliability and/or performance. 
In addition to experimental analysis, we  provide theoretical guarantees for the robustness of \sys{} in a decentralized linear regression setting under non-IID data distribution. Note that theoretical guarantees in non-IID settings are rarely found in the literature.
In summary, our contributions are as follows:
\begin{itemize}[leftmargin=*]
    \item We introduce \sys{}, a new peer-to-peer federated learning algorithm which is robust to a  variety of data and model poisoning attacks.
    \item \sys{} is shown to have the following desirable properties: (i) its robustness does not require the majority of the benign nodes over the compromised ones and is agnostic to the number of byzantine workers;
     (ii) \sys{} effectively adapts to and derives benefits from the non-IID datasets of benign clients. (iii) \sys{} is robust to numerous  poisoning attacks and shows superior performance compared to the state of the art defense methods.
    
    \item We theoretically prove the robustness of \sys{} in a decentralized linear regression setting under certain poisoning attacks.  
     We prove that the benign nodes in \sys{} learn the correct model parameter under Label-Flipping attacks. We then generalize our theoretical results to the General Random model poisoning attack. 
    \item Comprehensive, Proof-of-Concept evaluations on benchmark data from image classification (MNIST, FEMNIST, and CIFAR10) demonstrate the superiority of \sys{} over the existing defense methods, namely, Zeno, Trimmed Mean, Clipping, and FLTrust. The robustness is evaluated with respect to all six considered types of poisoning attacks. 
\end{itemize}



%% file: preliminaries.tex
\section{Background}
\label{sec:prelim}


\subsection{Variational Bayesian Learning}
\normalsize
  Variational Bayesian learning \cite{blundell2015weight} is the learning method used in this paper. In the variational Bayesian learning,
weight uncertainty is introduced in the neural network by learning a parameterized distribution over the weights (instead of just learning the weight values). 
More specifically, assume that $\theta$ is the weight vector of the neural network. We denote the parameterized posterior distribution over $\theta$ by $q(\theta|w)$, where $w$ is the parameter vector of the posterior distribution. Also, assume that at round $t$ of the training, we have a prior distribution over $\theta$, denoted by $\mathcal{P}_{t-1}(\theta)$. After observing the data batch $\mathcal{D}_t$, the posterior distribution parameters, $w$, are updated by minimizing the variational free energy loss function:
\begin{align}
\begin{small}
   \mathcal{F}(\mathcal{D}_t, w) \hspace{-0.05cm}= \hspace{-0.05cm}\text{\textbf{KL}}[q(\theta|w)||\mathcal{P}_{t-1}(\theta)]\hspace{-0.05cm}-\hspace{-0.05cm}\mathbb{E}_{q(\theta|w)}[\text{log} \mathcal{P}(\mathcal{D}_t|\theta)],
   \end{small}
    \label{eq:loss}
\end{align}

\normalsize
where $\mathcal{P}(\mathcal{D}_t|\theta)$ is the conditional likelihood of observing $\mathcal{D}_t$ given $\theta$. 

In the case of variational Gaussian  posteriors \cite{wang2022peer, al2020federated}, $q(\theta|w)$ is  a Gaussian distribution with mean $\hat{\theta}$ and  covariance matrix $\Sigma$ which is a diagonal matrix with  diagonal entries of $(\text{log}(1+\text{exp}(\rho)))^2$. Therefore, we have $w=(\hat{\theta},\rho)$.   

\subsection{Uncertainty Aware Model Averaging}
In a variational Bayesian federated learning setting such as \cite{wang2022peer, al2020federated},  clients incorporate the model  uncertainties in the aggregation process. In particular,  
each client aggregates the received model updates as follows.
 \begin{subequations}
 \begin{align}
    &({\Sigma}^{'i}_t)^{-1}=\sum_{j\in \mathcal{N}(i)}T^{ij}(\Sigma^j_t)^{-1}\\
&\hat{\theta}^{i}_t=\Sigma^{'i}_t\sum_{j \in \mathcal{N}(i)} T^{ij}(\Sigma^j_t)^{-1} \hat{\theta}^j_t\\
&{\Sigma}^{i}_t={\Sigma}^{'i}_t
\end{align}
 \label{eq:agg:prelim}
 \end{subequations}
where $w^i_t=(\hat{\theta}^i_t, \Sigma^i_t)$ is the Gaussian posterior belief  parameter, and $T^{ij}$ is the predetermined trust of client $i$ on client $j$. As can be seen in the above equation, the model elements with higher uncertainty (variance) are given smaller weights in the aggregation process.  

%% file: 2_Problem_Statement.tex
\section{Problem Statement}
\label{problemStatement}
\subsection{System Model}
We consider a peer-to-peer federated learning setting with  $N$  clients who are distributed over a time-varying  directed graph $G_t=(\mathcal{N},\mathcal{E}_t)$. $\mathcal{N}$ is the set of clients, where  each client corresponds to one node of graph $G_t$. An edge $(i,j)\in \mathcal{E}_t$ indicates that clients $i$ can communicate with client $j$ at time $t$. The adjacency matrix of the graph at time $t$ is denoted by $A_t$.
We denote the set of client $i$ together with its in-neighbors at time $t$ by $\mathcal{N}_t(i)$, and the set of out-neighbors of client $i$ at time $t$ is denoted by $\mathcal{N}_t^o(i)$.
Each client $i$ has access to a data set $\mathcal{D}^i$ 
consisting of data samples 
$(x^i_t,y^i_t)$. We assume the dataset of clients are non-IID with $x^i_t\in \mathcal{X}^i$ distributed according to $\mathcal{P}^i$. 
Clients learn the model parameter, $\theta$, according to a decentralized variational learning algorithm with Gaussian variational posteriors which will be explained in section \ref{sec:SABRE}. 

\subsection{Threat Model}
\label{sec:threat}
We consider an adversarial environment where a subset  of the nodes, $\mathcal{N}^c$, are  under  poisoning attacks by  attackers from outside the system.  While our analysis can  be applied to many types of poisoning attacks,  we focus on the following attacks.
 \begin{itemize}[leftmargin=*]
\item Data Poisoning Attacks:
\begin{itemize}[leftmargin=*]
 \item \textbf{Trojan}: A subset of the dataset of the compromised clients is added with a Trojan and labeled with a target label \cite{xie2020dba}.  This is a stealthy and colluding attack.
    \item \textbf{Label-Flipping}: The label of some classes are changed to a target class  \cite{tolpegin2020data}. 
\end{itemize}
  \item Model Poisoning Attacks:
    \begin{itemize}[leftmargin=*]
        \item \textbf{Bit-Flip}: In this attack, some of the bits in the binary representations of the model weights are flipped 
        \cite{xie2018generalized}.
        \item \textbf{General Random}: The attackers randomly choose some of the model weight elements  and  multiply them by a large number \cite{xie2018generalized}.
        \item \textbf{A Little is Enough}: In this attack, the adversaries place their model weights close to a number of  benign clients that are far form the mean to gain the majority power \cite{baruch2019little}. This type of attack is categorized as a colluding attack.
        \item \textbf{Gaussian}: The adversaries add Gaussian noise to the model updates \cite{pillutla2022robust}. 
    \end{itemize}
\end{itemize}


%% file: SABRE.tex
\section{\sys: A Robust  Federated Learning Framework} 
\label{sec:SABRE}
In this section, we describe \sys{}, our novel peer-to-peer  federated learning algorithm.
\sys{}  uses the variational learning method with Gaussian posteriors   explained in Section \ref{sec:prelim}. All of the steps of the  \sys{} algorithm are presented in  Algorithm \ref{alg:SABRE}.
\sys{} distinguishes itself from other frameworks by incorporating several innovative components, which are explained in the following. The most important component of \sys{} is the introduction of \textbf{Local and Social Models}.  In \sys{}, clients learn local  likelihood functions (also referred to as models or beliefs)  on the model parameters. In particular, clients train two models; a local model, and a social model. The local model is trained only using the local dataset based on the variational Bayesian learning method, while the social model is trained according to the federated learning algorithm that will be described later. The  social and local models are Gaussian likelihood functions with parameters denoted by $\bar{w}^i_t=(\bar{\theta}^i_t, \bar{\Sigma}^i_t)$ and $\hat{w}^i_t=(\hat{\theta}^i_t, \hat{\Sigma}^i_t)$, respectively, where $\bar{\theta}^i_t$ and  $\hat{\theta}^i_t$ are the means  of the social and local Gaussian belief of client $i$ on $\theta$ and $\bar{\Sigma}^i_t$ and $\hat{\Sigma}^i_t$ are their covariance matrices. The learning happens with the social model, while the local model has two critical applications: (1) Identifying  Compromised Clients, and (2) Introspection and Model Overwriting:  

\textbf{(1) Identifying  Compromised Clients:}
It is well established in the literature that  if a number of clients are under attack in a peer-to-peer federated learning algorithm,  the models of the benign clients can also become poisoned (due to aggregating their models with the malicious model updates). Therefore, these models can not be used as a ground truth to evaluate other clients' models and identify poisoned ones.  However, local models, trained solely on local datasets of benign clients, remain entirely clean.  
Therefore, they can be utilized as a reliable ground truth to identify compromised clients. Nonetheless, the local models suffer from low accuracy due to the limited size of the local datasets.
To address this limitation, Bayesian models prove advantageous by providing uncertainties in their model training. These uncertainties are incorporated into the process of evaluating received model updates from different clients. This enhances the evaluation process significantly, as clients primarily rely on elements of their models with higher certainty to flag malicious users. Consequently, this approach minimizes false positives, ensuring that benign clients are not erroneously flagged as malicious.  This is facilitated by the  novel \textbf{Uncertainty Aware Model Aggregation} method used in \sys{}, which is another important component of our method and is described in the following.

 Unlike \sysbase{} \cite{wang2022peer} that considers  predetermined time-invariant trust weights in Eq.\eqref{eq:agg:prelim},  we consider time-dependent trust weights, $T^{ij}_t$, which are defined to robustify the algorithm against  poisoning attacks. We define the trust weights in \sys{}, referred to as the bounded confidence trust weights,  such that  each client only aggregates the opinion of those with similar opinion to it, where two models are deemed similar if their element-wise distance is less than  a confidence bound. The confidence bound is determined by the uncertainty of the client over its own model, which is provided by the  Bayesian models. The trust weights are formally defined below.
\begin{definition}[Bounded Confidence Trust Weights]
\label{def:trust}
For each client $i$, the set $I(\bar{w}^{\mathcal{N}_t(i)}_t,\hat{w}^i_t)$, which is called the  confidence set of client $i$ at time $t$,  is defined as follows. 
\begin{align}
I(\bar{w}^{\mathcal{N}_t(i)}_t,\hat{w}^i_t)=\nonumber\{j \in \mathcal{N}_t&(i):  |(\hat{\theta}^i_t)_k-(\bar{\theta}^j_t)_k|\\&\leq \kappa\sqrt{(\hat{\Sigma}^i_t)_{k,k}}, \ \forall k \in \mathcal{K}\}
\end{align}
where $\mathcal{K}=\{1, \cdots K\}$ and $K$ is the dimension of the model parameter, and  $\kappa$ is a hyper parameter.  We define the trust weights as follows. 
\begin{align}
    T^{ij}_t=\frac{1}{|I(\bar{w}^{\mathcal{N}_t(i)}_t,\hat{w}^i_t)|}\mathbf{1}(j \in I(\bar{w}^{\mathcal{N}_t(i)}_t,\hat{w}^i_t)) 
    \end{align}
    \label{BCW}
\end{definition}
\vspace{-0.3cm}
 The hyperparameter $\kappa$ determines how strict  the aggregation rule is. If $\kappa$ is too large, poisoned  updates might  be aggregated and if $\kappa$ is too small, both benign and poisoned updates  might not be aggregated. We will see through our experiments that $\kappa$ can be easily tuned to have a robust algorithm that also performs well in benign settings. 

\textbf{(2) Introspection and Model Overwriting:} The local models are also used in an introspection process, in which a client will evaluate its  own social model.  
This is done  based on the bounded confidence measure described in definition \ref{def:trust} such that if a hypothetical client with the social belief of  client $i$ would not be in the bounded  confidence set of client $i$, then it will overwrite its social belief with its local belief. Notice that the model overwriting is done in the beginning rounds of the algorithm where the accuracy of the local model is still improving (before overfitting happens).

\begin{algorithm}[ht]
\small
\caption{\sys{}  Algorithm executed by client $i$}
\label{alg:SABRE}
\begin{algorithmic}
\STATE Input:    ${\Sigma}^i_0$, $\Sigma^{thr}$, $\kappa$, and $T_{max}$.  
\STATE Initialize $t=1$,  $\bar{w}^i_0=\hat{w}^i_0=(0,
\Sigma^i_0)$.
\WHILE {$\bar{\Sigma}^i_{t-1}> \Sigma^{thr}$ and $t<T_{max}$}
\STATE Receive  {$\mathcal{D}^i_t$} and compute $\hat{w}^{i}_t=(\hat{\theta}^{i}_t,\hat{\Sigma}^i_t)$ and $\bar{w}^{i}_t= (\bar{\theta}^{i}_t,\bar{\Sigma}^i_t)$,  by minimizing the variational free energy loss function \eqref{eq:loss},
starting from  $\hat{w}^i_{t-1}$ and $\bar{w}^i_{t-1}$, respectively. 

\STATE Share $\bar{w}^{i}_t$ with $\mathcal{N}^o_t(i)$ and receive {$\bar{w}^{j}_t$ for $j \in \mathcal{N}_t(i)$}.
\STATE Set 
    \begin{align*}
        I(\bar{w}^{\mathcal{N}_t(i)}_t,\hat{w}^i_t)=\{j \in \mathcal{N}_t(i) |&(\hat{\theta}^i_t)_k-(\bar{\theta}^j_t)_k|\\&\leq \kappa\sqrt{(\hat{\Sigma}^i_t)_{k,k}}, \ \forall k \in \mathcal{K}\}
    \end{align*}
\STATE Set   $T^{ij}_t=\frac{1}{|I(\bar{w}^{\mathcal{N}_t(i)}_t,\hat{w}^i_t)|}\mathbf{1}(j \in I(\bar{w}^{\mathcal{N}_t(i)}_t,\hat{w}^i_t))$. 
\STATE 
\begin{subequations}
\label{eq:SoBeAgg_alg}
    \begin{align}
(\bar{\Sigma}^{'i}_t)^{-1}=\sum_{j\in \mathcal{N}_t(i)}T^{ij}_t(\bar{\Sigma}^j_t)^{-1}\\
\bar{\theta}^{i}_t=\bar{\Sigma}^{'i}_t\sum_{j \in \mathcal{N}_t(i)} T^{ij}_t(\bar{\Sigma}^j_t)^{-1} \bar{\theta}^j_t
\end{align}
\end{subequations}
\STATE Set  $\bar{\Sigma}^{i}_t=\bar{\Sigma}^{'i}_t$.
\IF {$\exists k\in \mathcal{K} : \ |(\hat{\theta}^i_t)_k-(\bar{\theta}^i_t)_k|> \kappa \sqrt{(\hat{\Sigma}^i_t)_{k,k}}$}
\STATE Set $(\bar{\theta}^i_t)_k=(\hat{\theta}^i_t)_k$.
\ENDIF

\STATE $t=t+1$

\ENDWHILE
\end{algorithmic}
\end{algorithm}

%% file: 4_Learning_Robustness.tex
\section{Learning and Robustness Analysis}
\label{sec:learning}

In order to provide a thorough theoretical analysis of the learning in \sys{},  in this section, we consider a special case where the clients' models constitute of a single linear layer with parameter  $\theta^*\in\mathbb{R}^K$ (decentralized linear regression). We assume that the labels $y^i_t$ are generated according to the linear equation of $y^i_t=\langle\theta^*,x^i_t\rangle+\eta^i_t$, where $\langle\rangle$ denotes the inner product and we assume $\eta^i_t\sim N(0,\Sigma^i)$ is a Gaussian  noise. 
In order to simulate the non-IID datasets of clients,  we assume $x^i_t\in \mathcal{X}^i$ and  $\mathcal{X}^i\coloneqq \{x \in \mathbb{R}_+^K: x_k=0, \text{for } k \not\in \mathcal{K}^i\}$, where $\mathcal{K}^i\in \mathcal{K}$ is a subset of model parameter indices that user $i$ makes local observations on. The dataset of each client can be deficient for learning $\theta^*$. That is, we can have $\mathcal{K}^i\neq  \mathcal{K}$, for all or some $i \in \mathcal{N}$. Note that this model can represent both vertical ($\mathcal{K}^i\neq \mathcal{K}$) and horizontal ($\mathcal{K}^i= \mathcal{K}$) federated learning settings. 
For decentralized linear regression problem, the variational Bayesian learning updates will be simplified to a Kalman filter update on the parameters of the beliefs as described in Appendix equations \eqref{eq:PrBeUp_alg} and \eqref{eq:SoBeUp_alg}. 

In order to analyze the learning in \sys{}, we need to make the following connectivity assumption on the communication graph. 
\begin{definition}[Relaxed Connectivity
  Constraint]\label{def:rel_con}
If there exists a set $\{s_1, s_2, \cdots, s_t, \cdots\}$  with integers $s_t>0$, such that a graph with adjacency matrix $\bar{A}_t=\prod_{l=s_{t}}^{s_{t-1}}A_l$ is  strongly connected, 
then we say that the communication graph $A_t$ satisfies  the relaxed connectivity constraint.   
\end{definition}

Note that under the relaxed connectivity constraint, the communication graph over a cycle does not have to be connected and at some times it can be a disconnected graph. The intuitive explanation is that in order for the  learning to happen correctly, it is sufficient that the message of a client  is received by others after finite cycles and  although there might not be  a path between two clients at each cycle $t$, it is sufficient that a  path is formed between them in near future.  We also note that the relaxed connectivity constraint is similar to the B-strong connectivity condition introduced in \cite{nedic2015nonasymptotic}. 



In the next theorem, we show that in \sys{}, clients learn the true model parameter in a benign setting. 
\begin{theorem}[Learning by \sys{}]\label{thm:learn_bcp2pfl}
If agents learn according to \sys{}  algorithm and the communication graph  satisfies the relaxed connectivity constraint,  and if no agent is compromised, i.e., $\mathcal{N}^c=\emptyset$, then each agent $i$ learns the model parameter $\theta^*$ with mean square error that is decreasing proportional to $\frac{1}{t}$.
\end{theorem}
We note that the learning rate in \sys{}  algorithm is of the same order, $\frac{1}{t}$, as the learning rate in \sysbase{}. It means that modifications made to add robustness to the learning process have not compromised its performance in a benign setting. We also need to mention that the actual learning rate in \sys{}  is even higher than \sysbase{} as is observed through our experiments. 

We also note that the mean square error of the model parameter estimation decreasing proportional to $\frac{1}{t}$ indicates that the mean square error of the label predictions will also decrease proportional to  $\frac{1}{t}$. Thus, if $\frac{1}{t}<\delta$, for some small and arbitrarily chosen $\delta$, we have $\mathbb{E}[(y^i_t-\hat{y}^i_t)^2]<\xi \delta$, where $\hat{y}^i_t$ is the predicted label and $\xi$ is a constant.


We analyze the robustness of \sys{} under  a special case of the Label-Flipping data poisoning attack described in Section \ref{sec:threat}, in which the attackers poison the dataset of the compromised nodes by adding a bias $b$ to their data sample labels, $(y^i_t)'=y^i_t+b,\  \forall i \in \mathcal{N}^c$. Without loss of generality, we assume the attackers to different clients agree on a single bias $b$ to add to the data labels. 
 The second attack that we consider is the General Random model poisoning attack described in Section \ref{sec:threat}. 
 All of the proofs of the theorems in this section can be found in Appendix \ref{proofs}. 
 
In order to study the robustness of \sys{}, we need to state the following assumptions. 


\begin{assumption}[Sufficiency]
\label{ass:suff}
The collection of the datasets of benign clients is sufficient for learning the model parameter $\theta^*$. That is,  $\cup_{i \in \mathcal{N}/\mathcal{N}^c}\mathcal{K}^i=\mathcal{K}$.
\end{assumption}
\begin{assumption}[Relaxed Connectivity]
\label{ass:rel_con}
The communication graph of the benign nodes satisfies the relaxed connectivity constraint. 
\end{assumption}
\begin{assumption}[Joint Learning]
For every pair of benign node $i$ and compromised node $j$ that can communicate with each other at some time, we have $\mathcal{K}^i \cap \mathcal{K}^j \neq \emptyset$.
\label{ass:Jlearning}
\end{assumption}
The Sufficiency and Relaxed Connectivity assumptions are common assumptions needed for learning.  The Joint Learning assumption is made to ensure benign users can detect the compromised ones. In order for this detection to happen, a benign client has to be learning at least one common element of the model parameter with a compromised client to be able to evaluate its updates and detect it. 

\begin{theorem}[Robustness of \sys{}, Label-Flipping Attack]
In \sys{},  if nodes $\mathcal{N}^c$ are compromised by Label-Flipping attack, and if assumptions \ref{ass:suff}, \ref{ass:rel_con}, and \ref{ass:Jlearning} hold, then  the estimation of the benign users converge to $\theta^*$ with mean square error that is decreasing proportional to $\frac{1}{t}$. 
\label{thm:bcp2pfl_robust}
\end{theorem}

\begin{figure*}[t]
  \begin{center}
    \subfloat[Label-Flipping Attack]{\label{fig:label}\includegraphics[width=0.5\textwidth]{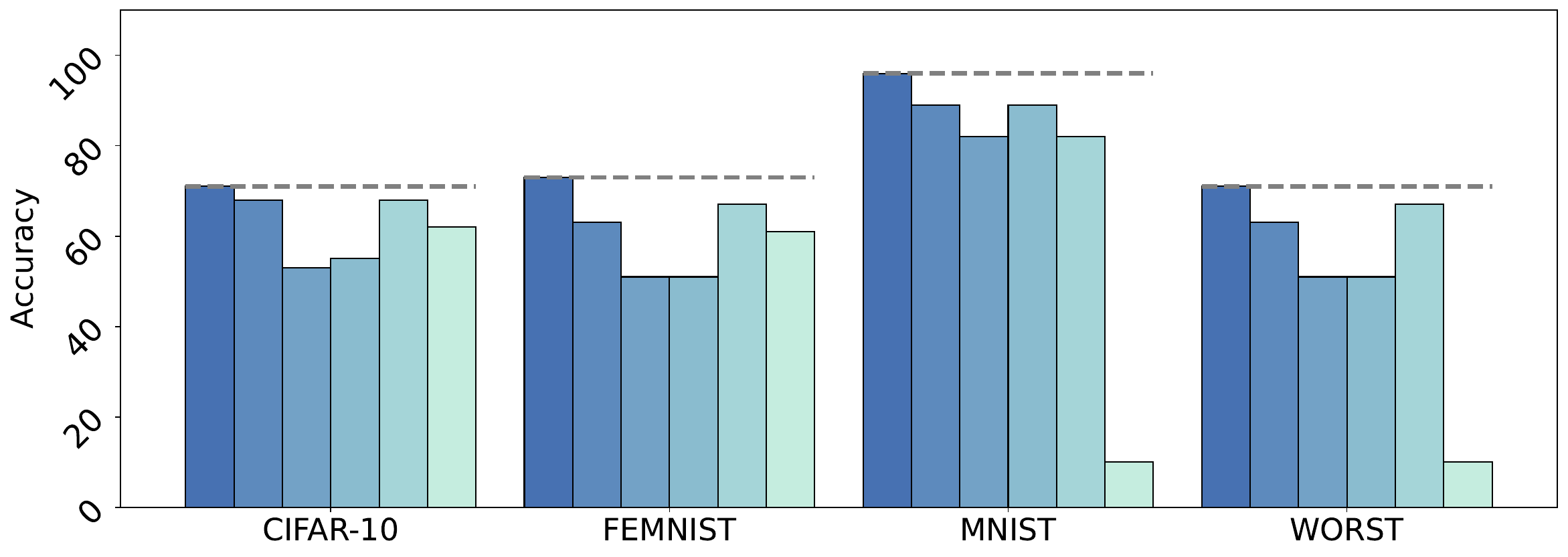}} 
     \subfloat[A Little is Enough Attack]{\label{fig:little}\includegraphics[width=0.5\textwidth]{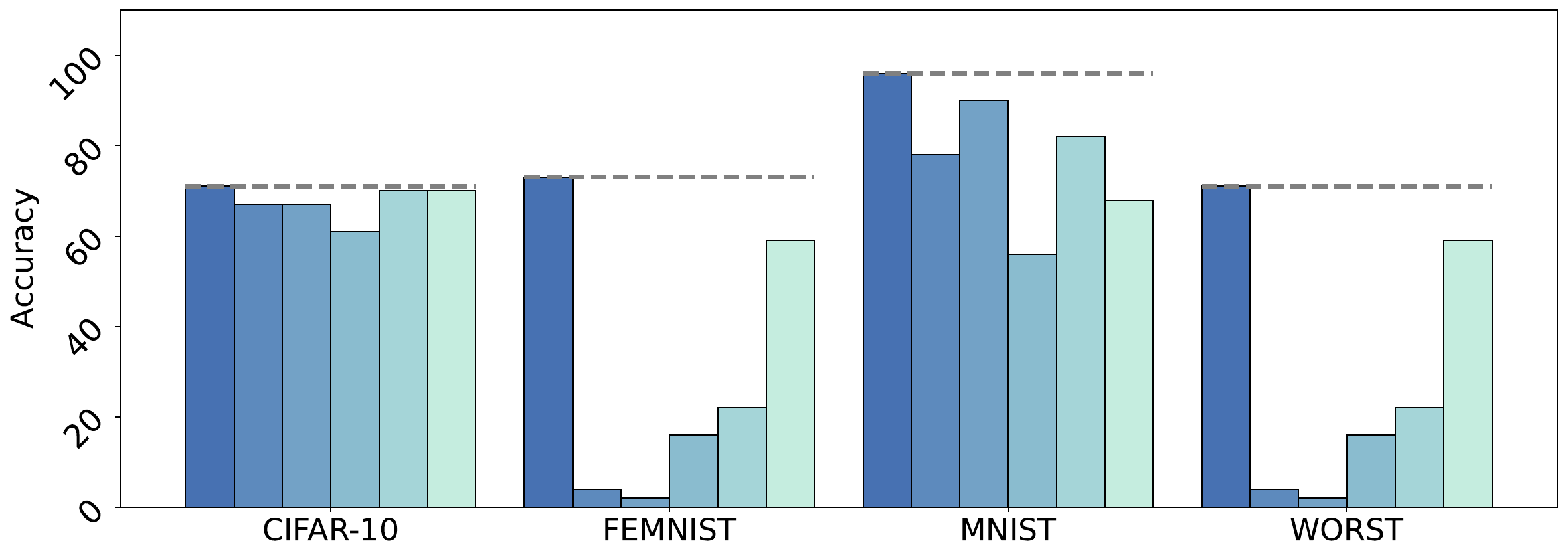}} \\
     \subfloat[Bit-Flip Attack]{\label{fig:bit}\includegraphics[width=0.5\textwidth]{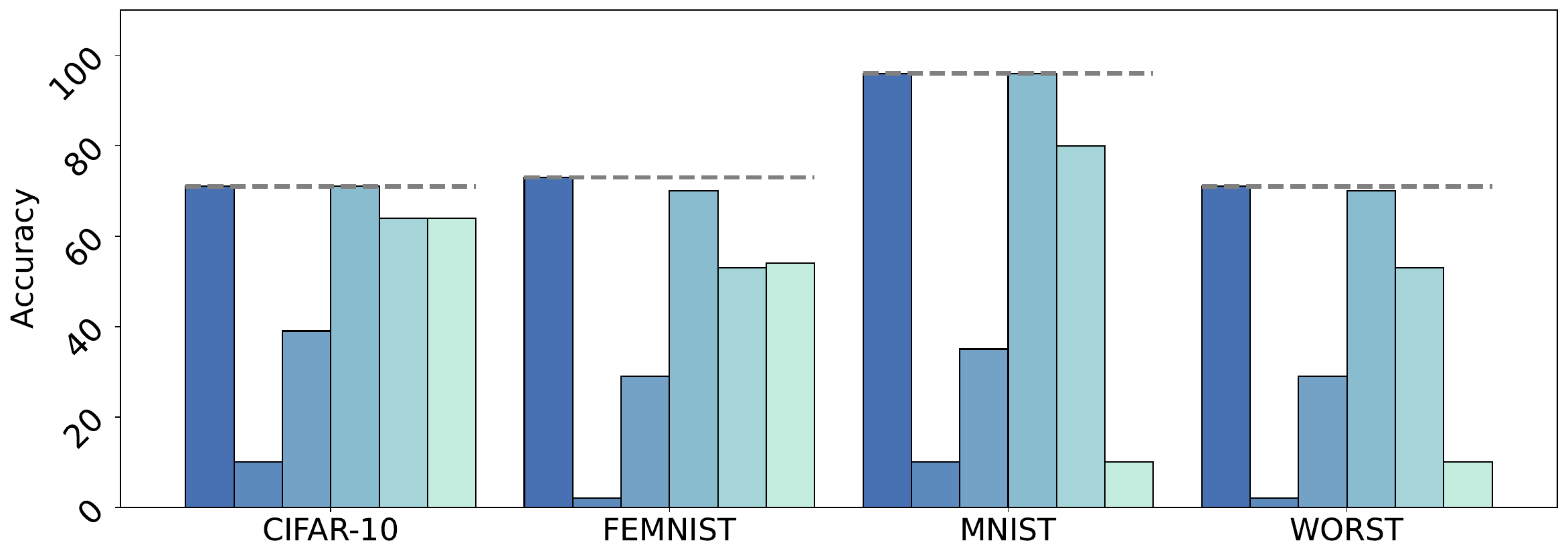}} 
    \subfloat[General Random Attack]{\label{fig:general}\includegraphics[width=0.5\textwidth]{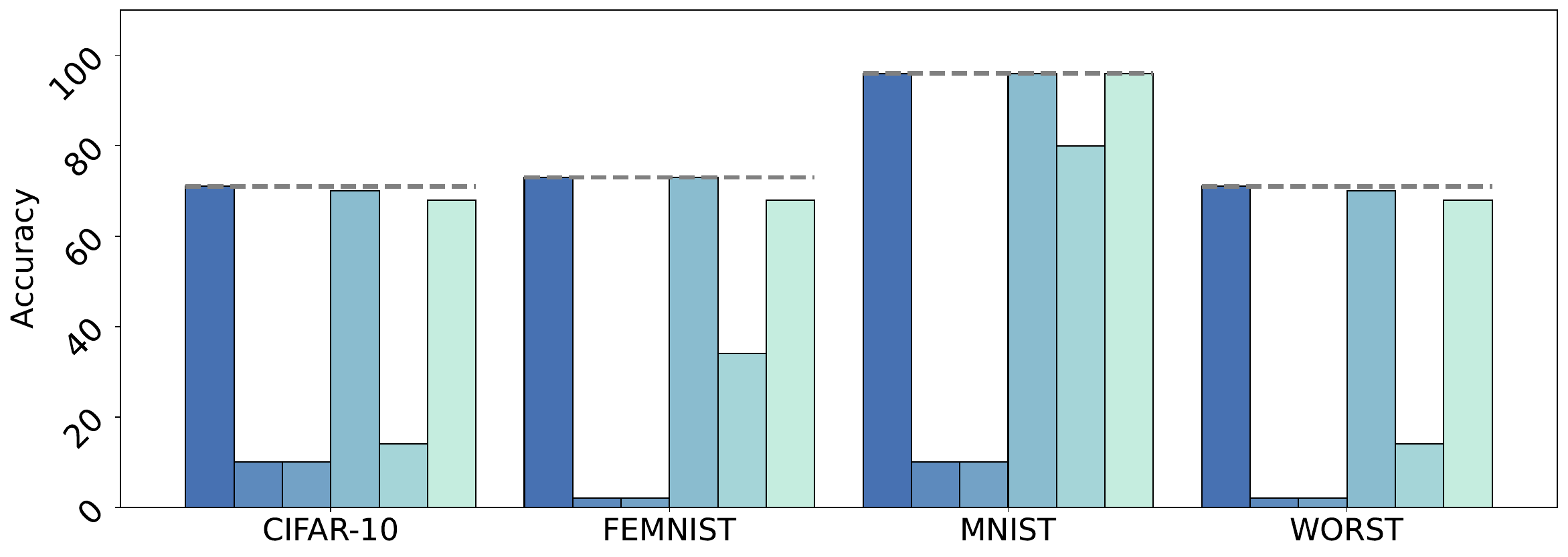}}\\
     \subfloat[Gaussian Attack]{\label{fig:gaussian}\includegraphics[width=0.5\textwidth]{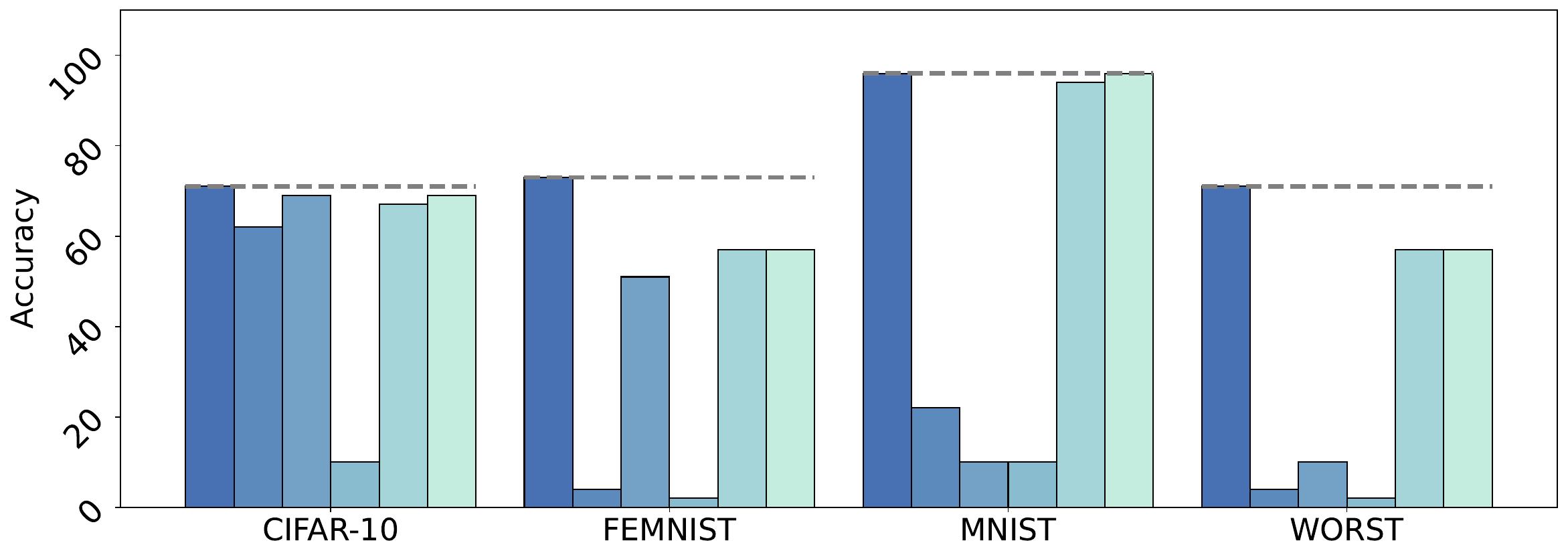}}
\subfloat{\label{fig:legend}\includegraphics[width=0.5\textwidth]{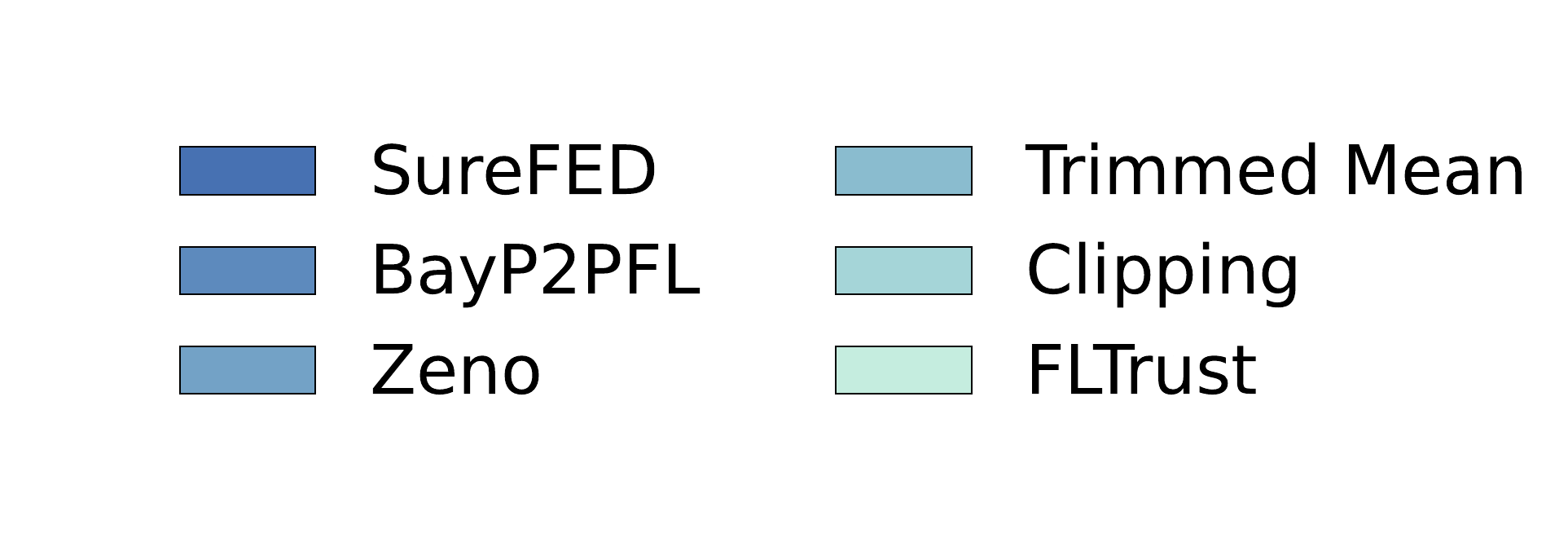}}
  \end{center}
  \vspace{-0.3cm}
  \caption{Final model accuracy of \sys{} and the other baselines  under different data and model poisoning attacks evaluated on CIFAR10, FEMNIST, and MNIST datasets. The dashed line shows the final accuracy in a benign setting. WORST plots correspond to the worst model accuracy of the methods across the three dataset.  \sys{} is the only framework that shows consistent robustness against all of the attacks and with all of the three datasets. }
  \label{fig:barplot_ns}
\end{figure*}

We need to note here that in order for the benign users to learn the true model parameter under  Label-Flipping  attack, we do not need to make any assumptions on the number of compromised nodes in each neighborhood nor do we need the majority of the benign nodes over the compromised ones. 
As long as  the above three assumptions hold, the benign nodes can learn the true model parameter even if all but one of their neighborhood are compromised.

In the next theorem, we will generalize our results to the General Random model poisoning attacks. This generalization is done by taking into account the fact that the models that are widely used in practice are huge, and this will make it almost impossible for the attackers to remain undetected in \sys{}.

 \begin{theorem}[Robustness of \sys{}, General Random Attack]
In \sys{},  if nodes $\mathcal{N}^c$ are compromised by the General Random model poisoning attack, and if  assumptions  \ref{ass:suff} and \ref{ass:rel_con} hold,  then  with probability  of almost 1, the estimations of the benign users converge to $\theta^*$ with mean square error that is decreasing proportional to $\frac{1}{t}$.  
\label{thm:bcp2pfl_robust_gen}
\end{theorem}
Note that we do not need the Joint Learning assumption in the above theorem. This theorem offers a probabilistic guarantee on the robustness of \sys{} with a probability that is almost 1 when the size of the model is large (see the proof in Appendix for more detail). 

%% file: 5_Experiments.tex
\section{Experiments}
\label{sec:experiments}
In this section, we present our experimental results for  a decentralized image classification setting. We consider a network of 50 clients where 40\% of them are under data and model poisoning attacks. 

\textbf{Datasets:}
 We evaluate the performance of \sys{} on three image classification datasets, namely, MNIST and FEMNIST from LEAF~\cite{caldas2018leaf}, and CIFAR10 \cite{krizhevsky2009learning}. The details of these datasets  are summarized in Tab.~\ref{tab:dataset} in the Appendix. 


\textbf{Baselines:}
We compare the performance of \sys{} with the state of the art defense methods for federated learning in all of the three categories that were described in the introduction. Specifically, we compare with  Trimmed Mean \cite{yin2018byzantine}, and Clipping \cite{bagdasaryan2020backdoor} from the first, Zeno \cite{zhao2016bounded} form the second, and FLTrust \cite{cao2020fltrust} from the third category.   
We also included \sysbase{} as a baseline to show the performance of a peer-to-peer federated learning using Bayeisan models without any defense mehtods. 



\begin{figure*}[!ht]
  \begin{center}
    \subfloat[\scriptsize Label-Flipping Attack]{\label{fig:femnist:label}\includegraphics[width=0.2\textwidth]{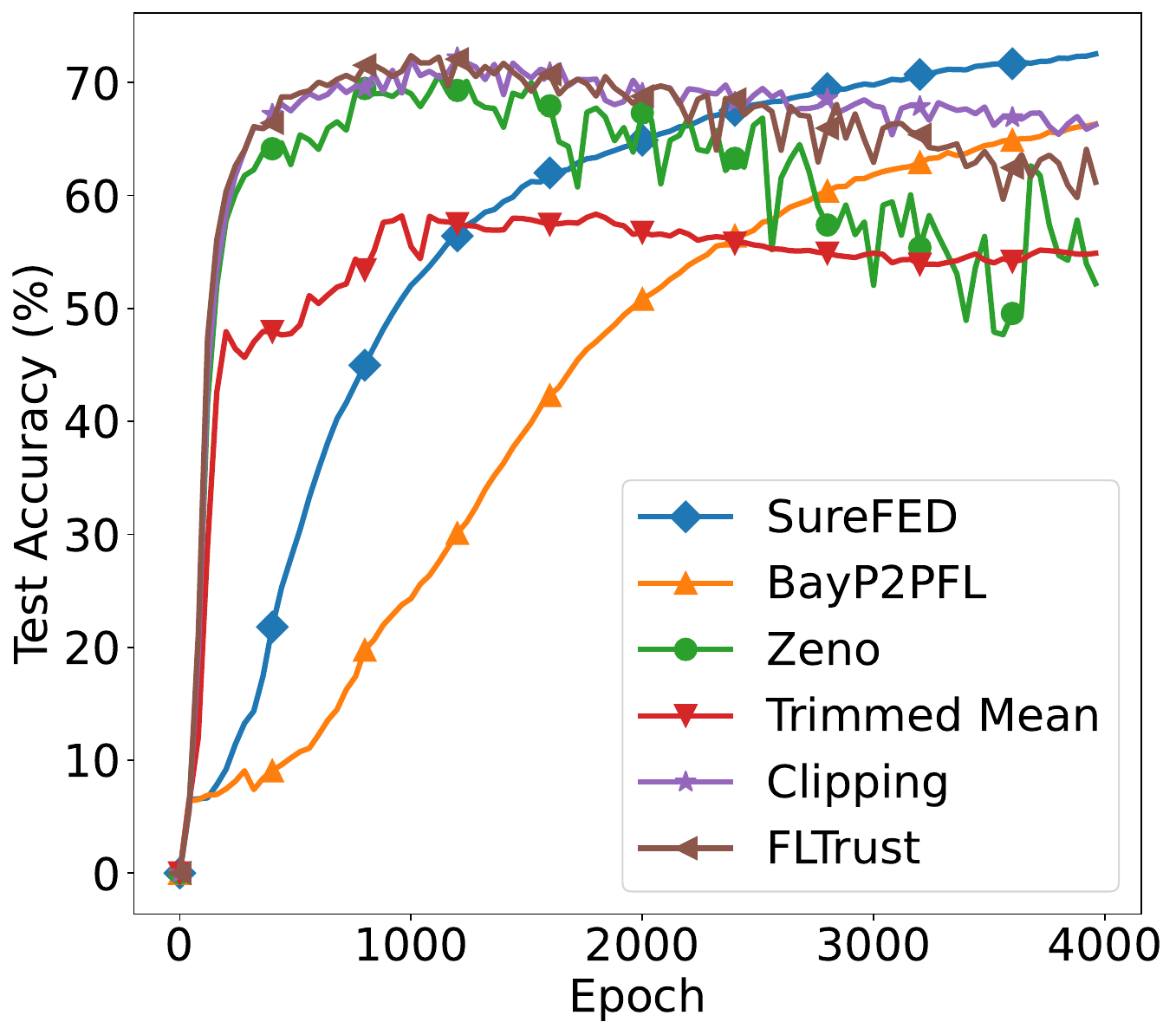}} 
     \subfloat[\scriptsize A Little is Enough Attack]{\label{fig:fmnist:little}\includegraphics[width=0.2\textwidth]{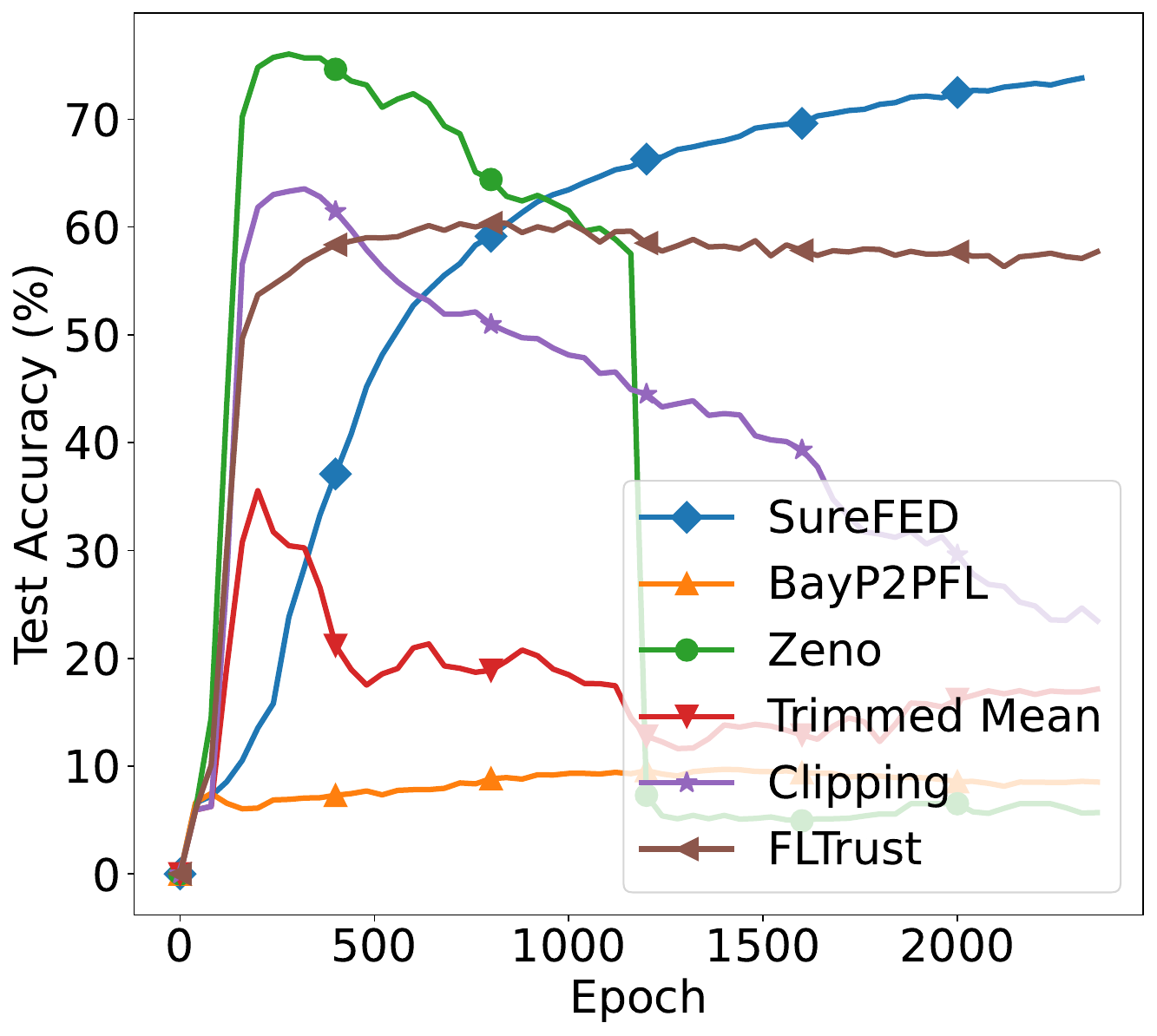}} 
     \subfloat[\scriptsize Bit-Flip Attack]{\label{fig:feminist:bit}\includegraphics[width=0.2\textwidth]{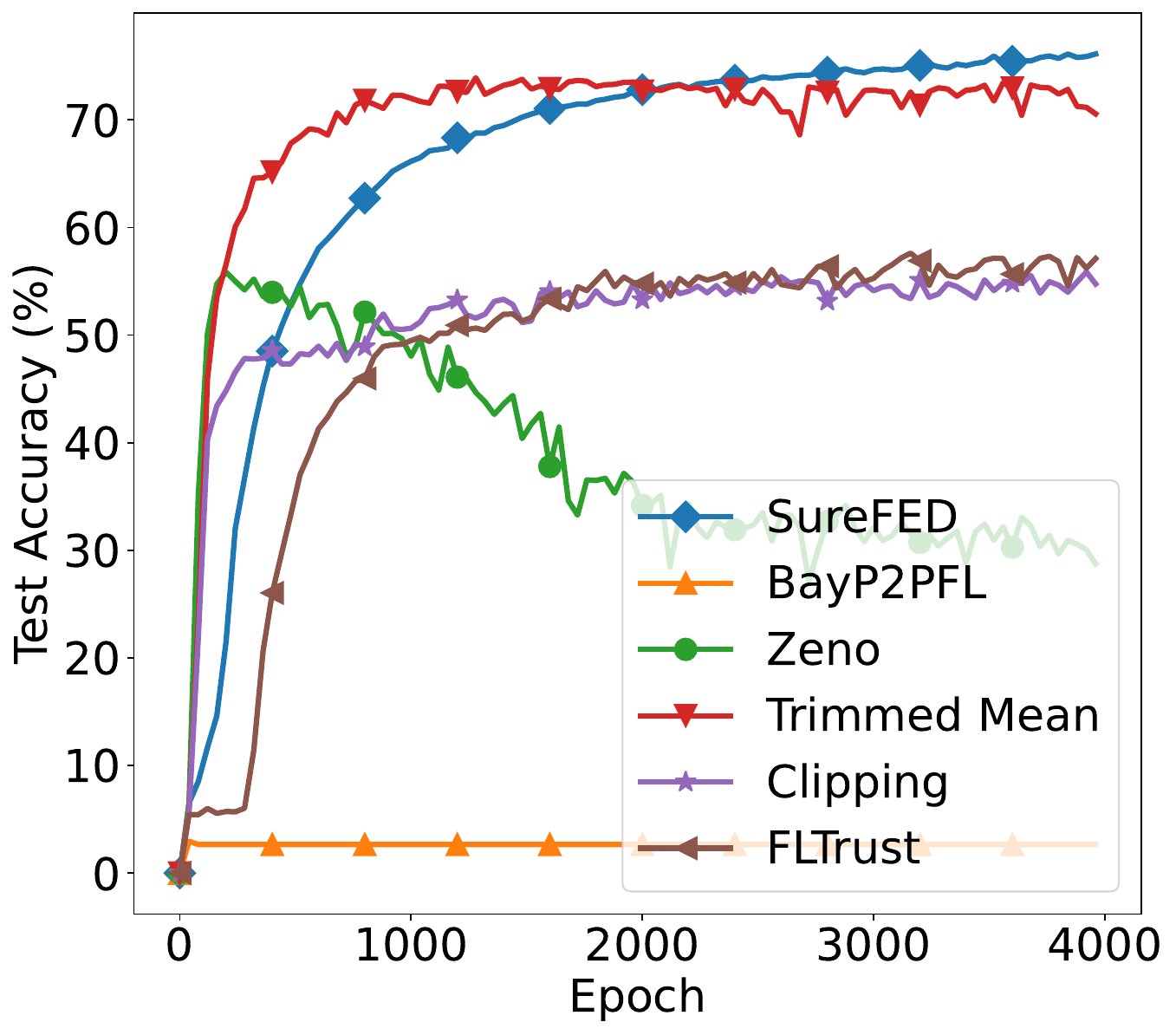}} 
    \subfloat[\scriptsize General Random Attack]{\label{fig:femnist:general}\includegraphics[width=0.2\textwidth]{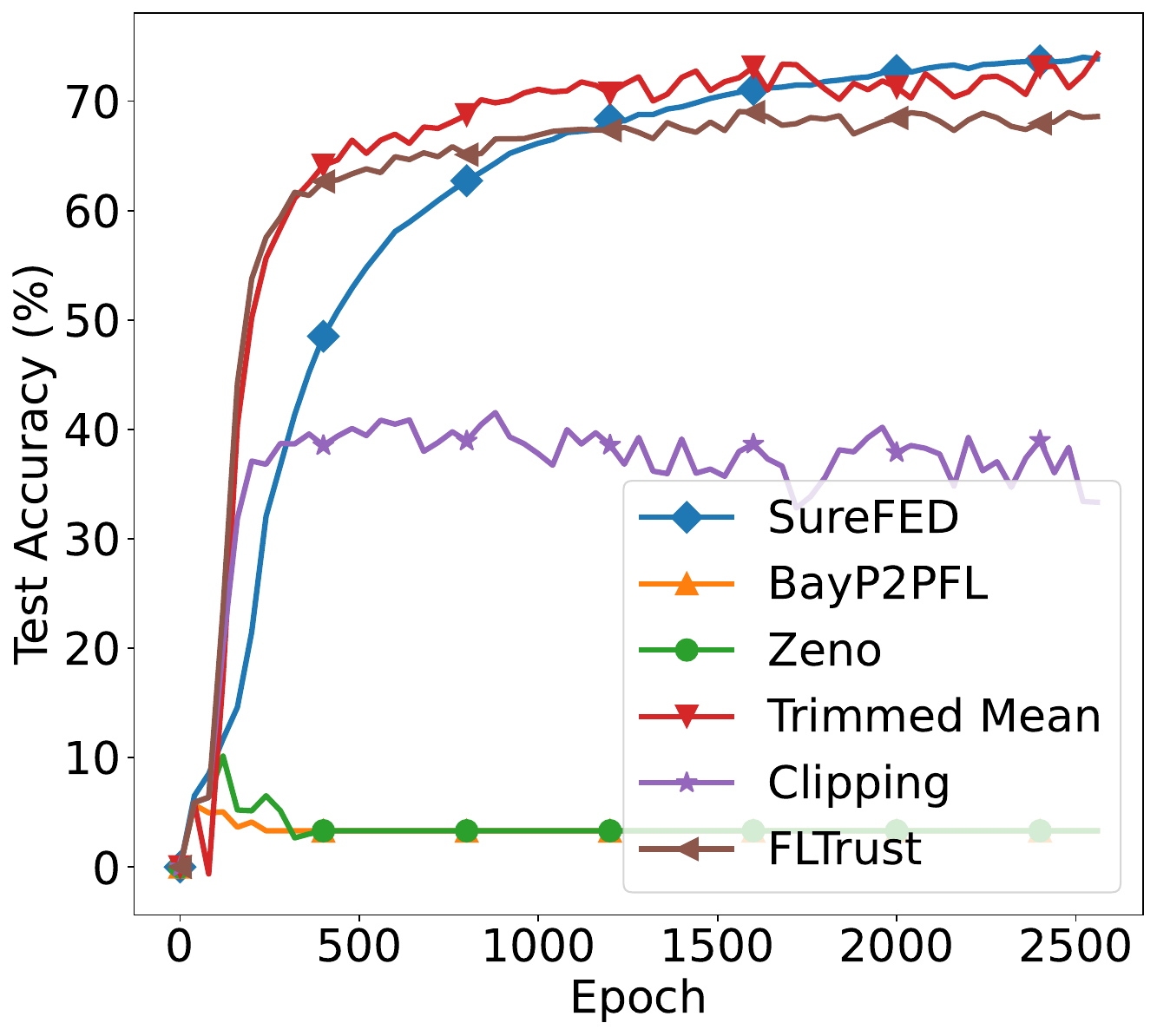}}
     \subfloat[\scriptsize Gaussian Attack]{\label{fig:femnist:gaussian}\includegraphics[width=0.2\textwidth]{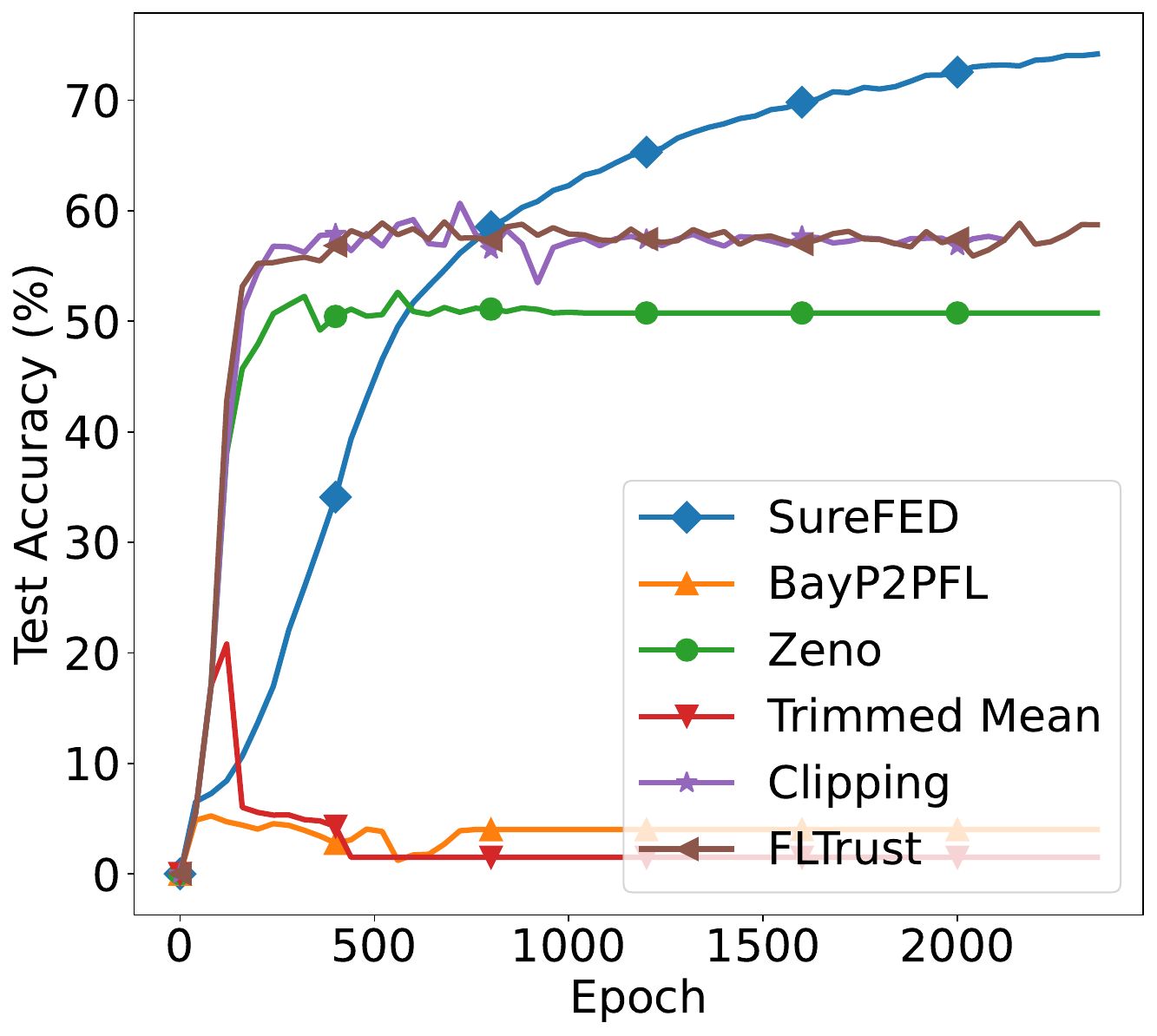}}
  \end{center}
  \caption{Accuracy plot of \sys{} compared with \sysbase{}, Zeno, Trimmed Mean, Clipping, and FLTrust defense methods under different data and model poisoning attacks evaluated on FEMNIST dataset. }
  \label{fig:femnist}
\end{figure*}

\textbf{Attack Details:} We consider six types of colluding and non-colluding data and model poisoning attacks that were described in  Section \ref{sec:threat}. 
For the Label-Flipping, Bit-Flip, and A Little is Enough attacks, we followed the experimental settings from Karimireddy~\textit{et.al.}~\cite{karimireddy2020learning}. For the Trojan attack, we followed the settings from DBA~\cite{xie2020dba}. For the General Random attack, we followed the settings from Xie~\textit{et.al.}~\cite{xie2018generalized}. The Gaussian attack is based on the settings of \cite{pillutla2022robust}.




\textbf{Evaluation Metrics:}
For the non-stealthy attacks of Label-Flipping, Bit-Flip, General Random, Gaussian, and A Little is Enough attacks that reduce the model accuracy, we use \textbf{Test Accuracy} measured on the benign test dataset to measure the robustness of different frameworks. The higher the test accuracy, the more robust the framework is against the considered attacks.  

Trojan attack is a stealthy attack that should not affect the test accuracy to remain stealthy. For Trojan attack, we evaluate \textbf{Main Task Accuracy (MA)} (which is the same as test accuracy on clean test dataset) to measure the stealthiness of the attack. The main task accuracy needs to remain high under Trojan attack to ensure the stealthiness of the attack. We also evaluate  \textbf{Backdoor Accuracy (BA)}, which is the percentage of Trojaned test samples that are successfully labeled with the Trojan target label (it is also referred to as attack success rate).  The lower the backdoor accuracy, the more robust the framework is to Trojan attack. 


\subsection{Results}

\noindent\textbf{Performance under non-stealthy attacks}: 
  In Fig.\ref{fig:barplot_ns}, we show the final model accuracy of \sys{} and other baselines under Label-Flipping, A Little is Enough, Bit-Flip,  General Random, and Gaussian  poisoning attacks for all of the three datasets of MNIST, FEMNIST, and CIFAR10. 
   Note that these attacks are non-stealthy and a poisoned model will have a lower test accuracy. Tables with the final model accuracy numbers are provided in the Appendix (tables \ref{tab:mnist:table}, \ref{tab:femnist:table}, and \ref{tab:cifar10:table}). In Fig.\ref{fig:femnist}, we also show the plot of test accuracy w.r.t. the training epochs  for FEMNIST dataset.
 The test accuracy plots of MNIST and CIFAR10 can be found in Appendix figures \ref{fig:femnist} and \ref{fig:cifar10}.
  As can be seen from the  plots in Fig.\ref{fig:barplot_ns}, \sys{} is consistently robust against all of the five poisoning attacks. Its accuracy surpasses that of the other five baselines, matching the benign final model accuracy of 96\% for MNIST, 73\% for FEMNIST, and 71\% for CIFAR10 (benign model accuracy is computed by training \sys{} with no attackers).  
 In contrast, the other baselines exhibit varying degrees of robustness against certain attacks,  with lower accuracy than \sys{}, and all of the baselines experience significant failures for particular attacks and datasets. In Fig.\ref{fig:worst}, we show the lowest model accuracy observed for each method across different datasets and attacks, highlighting the worst-case scenario for each framework. 
 Notably, \sys{} demonstrates robustness across all examined datasets and attacks, whereas the other methods exhibit poor performance in at least one dataset and attack, with an average accuracy of 6\%.


\begin{figure}[!ht]
  \centering
  \includegraphics[width=0.95\columnwidth]
  {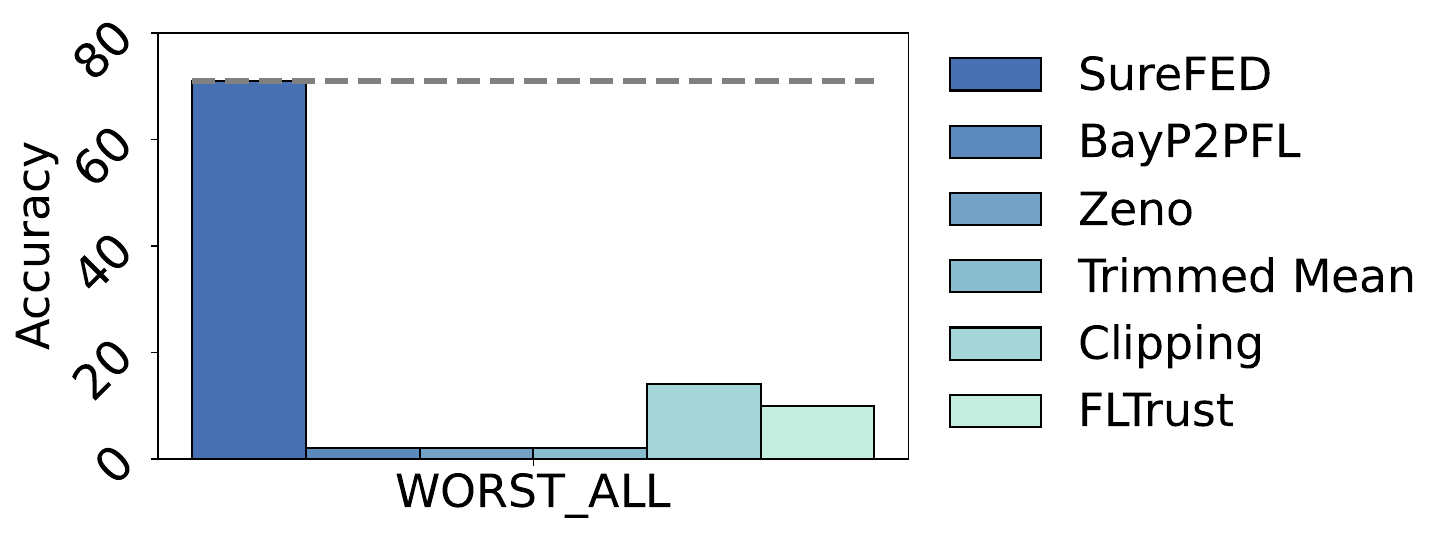}
  \caption{The worst final model accuracy of \sys{} and other baselines across different  poisoning attacks and datasets of CIFAR10, FEMNIST and MNIST. \sys{} is the only method that shows consistent robustness against all of the attacks and  and with all of the three datasets.}
 
  \label{fig:worst}
\end{figure}

\noindent\textbf{Performance under Trojan attack (stealthy attack)}:
For Trojan attack, we report the final Main Task Accuracy and Backdoor Accuracy of \sys{} and the other frameworks in Fig.\ref{fig:trojan}. As is expected from a stealthy Trojan attack, the main task accuracy needs to remain high for all of the frameworks and this is confirmed in Fig.\ref{fig:trojan}. The backdoor accuracy, however, shows the attack success rate and the lower this number is for a framework, the more robust it is to Trojan attacks. As can be seen  in Fig.\ref{fig:trojan}, the backdoor accuracy of \sys{} is 24\% and 23\% for MNIST and FEMNIST datasets, respectively, while the other frameworks show 100\% backdoor accuracy for  MNIST and around 80\% for FEMNIST dataset. 
\begin{figure}[!ht]
    \centering
\includegraphics[width=\columnwidth]{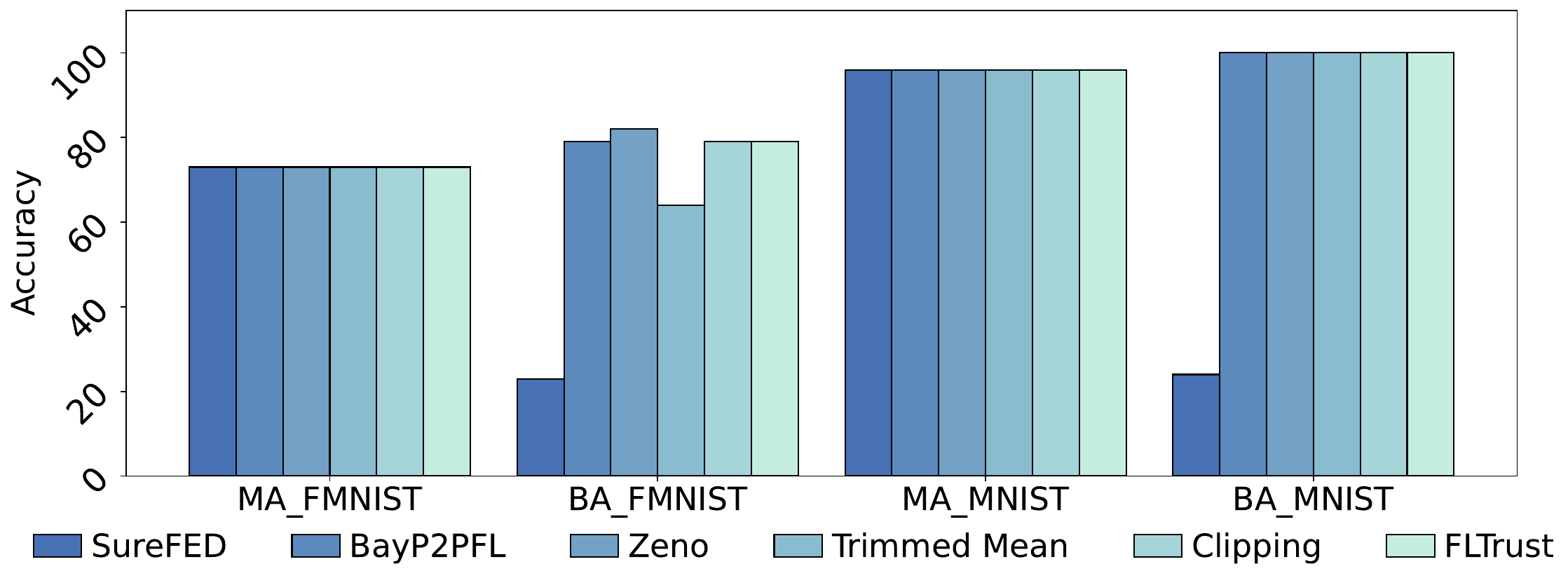}
    \caption{Main Task Accuracy (MA) and Backdoor Accuracy (BA) of \sys{} and the other baselines under Trojan  attack.  \textbf{High MA} indicates the success of  the Trojan attack in maintaining its stealthiness (MA of all baselines should be high). \textbf{Low BA} indicates success in defending against the Trojan attack. \sys{} is the only method with low BA, indicating its robustness against Trojan attacks. }
    \label{fig:trojan}
\end{figure}

\subsection{Ablation Studies}
We have done three types of ablation studies on \sys{}. The first study is on  the percentage of compromised clients in the network. As explained before, \sys{} does not require the majority of benign clients over the compromised ones to be robust against data and model poisoning attacks. This result is confirmed in Fig.\ref{fig:mnist:maj} where we show the plot of final model accuracy of \sys{} and the other frameworks w.r.t. the percentage of compromised clients. 
It can be seen that  \sys{} shows a robust behavior for all of the adversary percentages. Also note that \sys{} is agnostic to the number of adversaries. 

\begin{figure}[H]
    \begin{minipage}[c]{0.6\columnwidth}
\includegraphics[width=\textwidth]{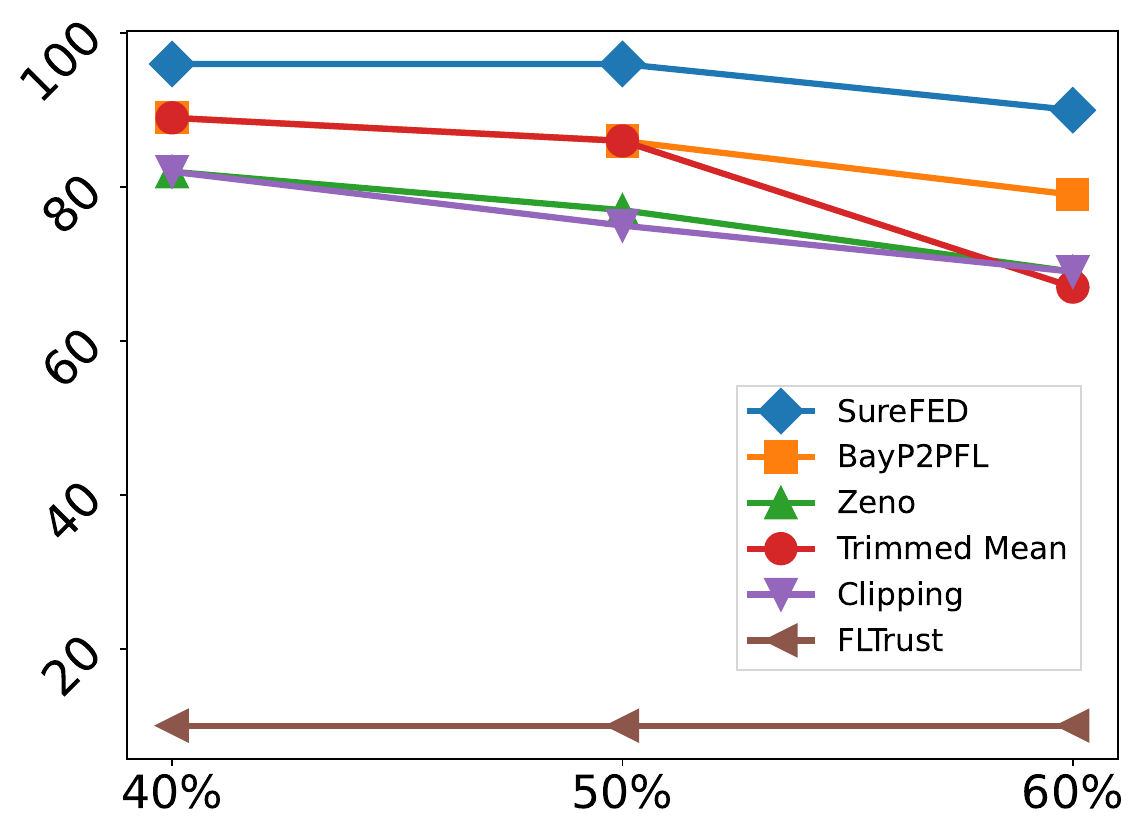}
    \end{minipage}
    \hfill
    \begin{minipage}[c]{0.39\columnwidth}
        \caption{Final Model Accuracy of \sys{} and the other baselines for MNIST dataset under Label-Flipping  attack and varying number of compromised clients percentages. }
        \label{fig:mnist:maj}
    \end{minipage}
     \vspace{-10pt}
\end{figure}

The second ablation study is on the communication network between the clients. It was stated in Section \ref{sec:learning} that for the learning and robustness in \sys{} to happen, we need the communication graph to satisfy the relaxed connectivity constraint defined in Def.\ref{def:rel_con}. Such graphs could be incomplete and time-varying.   In Table \ref{tab:cifar10:graph}, we provide the  results of the final model accuracy  with incomplete and time varying graphs and under Label-Flipping attack. This result  confirms that \sys{} is robust against the considered attack even when the communication graph is incomplete and time-varying. 

\begin{table}[h]
  \centering
  \small\addtolength{\tabcolsep}{-2pt}
  \begin{tabular}{|c|c|c|c|}
    \hline
    Graph  & Complete & Incomplete & Time Varying \\
    \hline
   \textbf{ \sys{}' Accuracy}  & \textbf{73\%}  & \textbf{73\%} & \textbf{73\%} \\  \hline
  \end{tabular}
  \caption{\sys{}'s  Model Accuracy under Label-Flipping attack with FEMNIST dataset and  different communication  graphs. The incomplete graph is constructed by randomly  dropping 20$\%$ of edges. The dropped edges are renewed every 100 epochs to construct the time-varying graph.}
  \label{tab:cifar10:graph}
\end{table}

Our third ablation study is on the choice of clients that are under attack. We consider two extremes, where in one, the clients with the best dataset quality are under attack and in the other, the worst ones are under attack. The quality of the clients' datasets is evaluated by their best local model accuracy. We compare the two extremes with a case where  we randomly choose clients to be under attack. For these experiments, we consider Label-Flipping attack and FEMNIST dataset.  In Table  \ref{tab:clientchoose} we see that the performance of \sys{} is not affected by how the clients are chosen to be under attack.

\begin{table}[h]
  \centering
  \small\addtolength{\tabcolsep}{-2pt}
  \begin{tabular}{|c|c|c|c|}
    \hline
    Compromised Clients  & Best  & Worst & Random \\
    \hline
   \textbf{ \sys{}' Accuracy}  & \textbf{73\%}  & \textbf{73\%} & \textbf{73\%} \\  \hline
  \end{tabular}
  \caption{\sys{}'s Model Accuracy under Label-Flipping attack and FEMNIST dataset, where different sets of clients are under attack. Best and worst refer to the clients with the best and worst local model accuracy, respectively. } 
  \label{tab:clientchoose}
\end{table}

%% file: 7_Discussion.tex
\section{Conclusion}
\label{sec:disc}
In this work, we presented \sys{}, which is a novel robust federated learning framework based on uncertainty quantification. \sys{} was designed to address the vulnerabilities of the existing defense methods by effectively using the local information of clients to concurrently train clean local models alongside social models. The local models can be used as a ground truth to  evaluate the received model updates in an uncertainty-aware manner, and also to conduct introspection, making sure the learning is done correctly. These crucial components empower \sys{} to withstand various data and model poisoning attacks, including colluding A Little is Enough attack and stealthy Trojan attacks. \sys{} demonstrates superior performance compared to the state of the art defense methods for federated learning, achieving model accuracies that match the benign  training accuracy, while being agnostic to the number of adversaries and not requiring the majority of benign clients over the compromised ones.

%% file: Appendix.tex
\section*{Appendix}

\subsection{Experiment Setup}
Our experiments are done with Python 3.9.7 and benchmarked on Linux Ubuntu 20.04.5 platform. Our workflow is built upon PyTorch~\cite{pytorch} version 1.9.0 and trained on four NVIDIA TITAN Xp GPUs each with 12 GB RAM, and 48 Intel(R) Xeon(R) CPUs with 128 GB RAM.

\subsection{Dataset Details}
In Table~\ref{tab:dataset}, from left to right, we show the label number, average training and test dataset size per client.
\begin{table}[ht]
  \centering
  \vspace{-5pt}
  \begin{tabular}{|c|c|c|c|}
    \hline
    Name  & Label & Avg. Training Data  & Avg. Test Data\\
    \hline
    MNIST & 10 &101 &12\\ \hline
    FEMNIST & 62 & 197 & 22 \\\hline
     CIFAR-10 & 10 & 2560 & 400 \\\hline
  \end{tabular}
  \caption{Dataset Statistics\label{tab:dataset}}
  \vspace{-5pt}
\end{table}

MNIST dataset consists of images of handwritten digits, and FEMNIST consists of handwritten digits and lower/upper case alphabets by different writers. The datasets are divided between clients such that each client only observes one writer's images to create a non-IID setup. CIFAR-10 is the biggest benchmark studied in the secure FL literature~\cite{ghodsi2023zprobe}. It consists of 60k 32x32 color images in 10 classes (50k for training and 10k for testing), with 6000 images per class. The dataset is divided randomly and evenly, where each client only observes part of the dataset.

\subsection{Model Training Details}
The architecture of the Bayesian neural network (BNN) and deep neural network (DNN) models used in this paper are summarized in Tab.~\ref{tab:modelarchi}. For fair comparison, both models consist of two convolution layers and two linear layers. The output dimension, label\_num, in layer Linear2 is set to 10 in MNIST dataset and 62 in FEMNIST dataset. The BNN is implemented according to the variational Bayesian learning model in \cite{blundell2015weight} and the federated version in \cite{wang2022peer}.
For CIFAR-10, We use the ResNet-20~\cite{he2016deep} as the backbone model for classification. The ResNet-20 is converted to the Bayesian neural network using Bayesian-Torch~\cite{krishnan2022bayesiantorch}.


We set the maximum number of epochs for training to 8000. The batch size is set to 5, and each client trains on 5 batches in one epoch  and then  sends its model update to its peers for aggregation. For BNN models, we use an AdaM optimizer with learning rate set to  0.001. Other BNN implementation details are the same as \sysbase{}~\cite{wang2022peer}. For Trimmed Mean and Clipping, we followed similar settings as Karimireddy~\textit{et.al.}~\cite{karimireddy2020learning} 
and set the learning rate to 0.01 with momentum set to 0.9. For Zeno, we followed similar settings as Xie~\textit{et.al.}~\cite{xie2019zeno} and set the learning rate to 0.1.
Note that, we choose SGD over AdaM optimizer to train DNN models due to its better convergence and accuracy.

For \sys, we set the hyperparameters, $\kappa$, to 2, and stop training the local model when its accuracy starts to drop in order to avoid overfitting local training data. 
For Zeno and Trimmed Mean, by default, we exclude the number of clients that have been compromised from the aggregation. \sys{},  Clipping, and FLTrust are agnostic to the number of compromised clients in the network, whereas Zeno and Trimmed Mean require this information during the training. 


\begin{table}[!ht]
  \centering
  \begin{tabular}{|c|c|c|c|}
    \hline
    Layers  & Patch Size/Stride & Output & Activation\\
    \hline
    BNN Conv1 / DNN Conv1  & 5$\times$5/1 & 6$\times$28$\times$28 & ReLU\\\hline
    Max Pooling1 & 2$\times$2/2 & 64$\times$1$\times$1 & -\\\hline
    BNN Conv2 / DNN Conv2& 5$\times$5/1 &16$\times$14$\times$14 &  ReLU\\\hline
    Max Pooling2 & 2$\times$2/2 & 64$\times$1$\times$1 & -\\\hline
    BNN Linear1 / DNN Linear1& 784 & 120 & ReLU\\\hline
    BNN Linear2 / DNN Linear2  & 120 & label\_num & Softmax\\\hline
  \end{tabular}
  \caption{Model Architecture\label{tab:modelarchi}}
\end{table}



\subsection{Experiment Results}
In tables \ref{tab:mnist:table}, \ref{tab:femnist:table}, and \ref{tab:cifar10:table} we present the final model accuracy of \sys{} and the other baselines under the five considered non-stealthy attacks and with the three datasets of MNIST, FEMNIST, and CIFAR10. The main task accuracy and backdoor accuracy of \sys{} and the other baselines under stealthy Trojan  attack are provided in Table \ref{tab:trojan}. The plot of test accuracy w.r.t. the epochs of all methods is provided in figures \ref{fig:mnist} and \ref{fig:cifar10}.

\begin{figure*}[!ht]
  \begin{center}
     \subfloat[\scriptsize Label-Flipping Attack]{\label{fig:mnist:label}\includegraphics[width=0.2\textwidth]{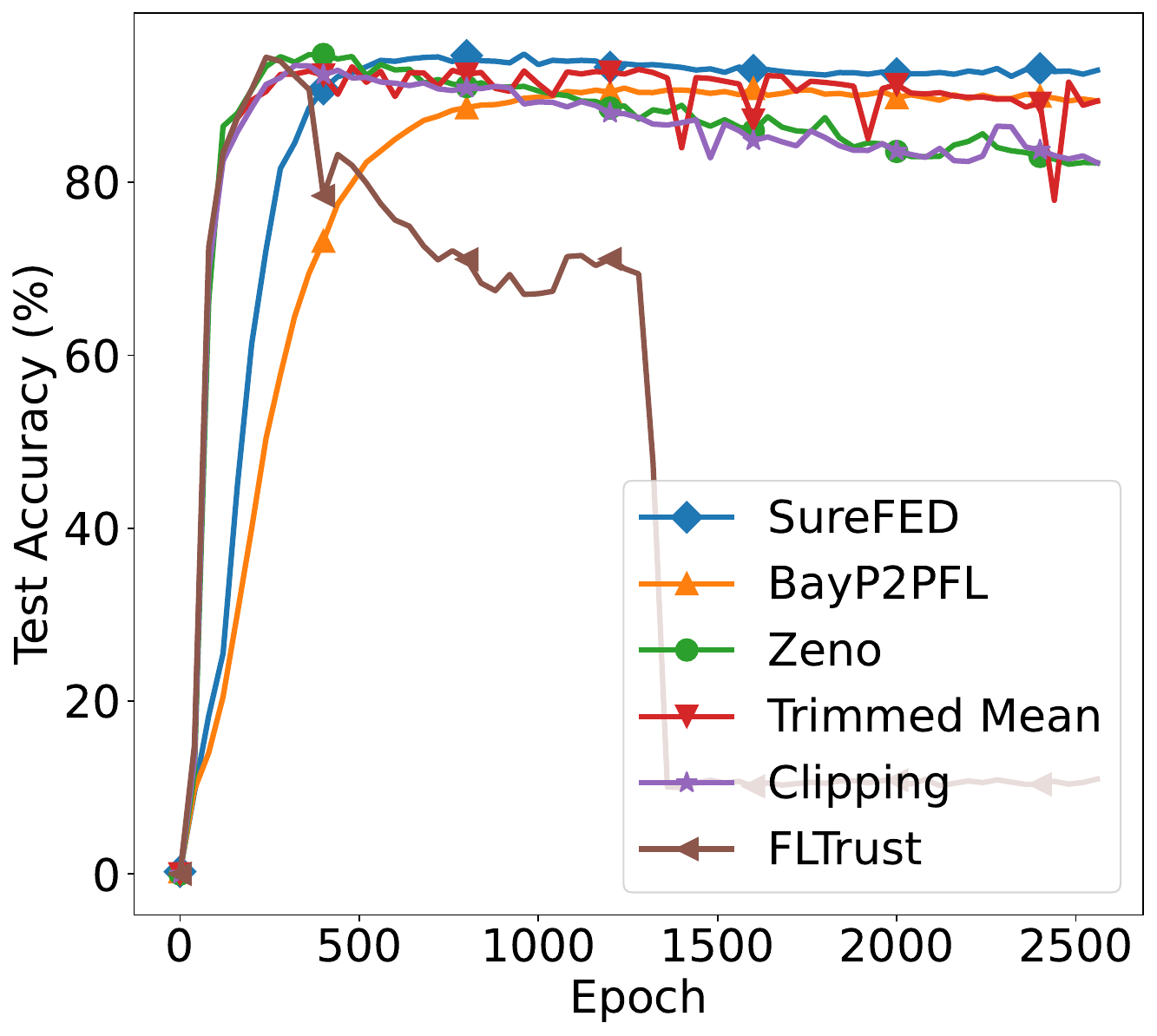}} 
    \subfloat[\scriptsize A Little is Enough Attack]{\label{fig:mnist:little}\includegraphics[width=0.2\textwidth]{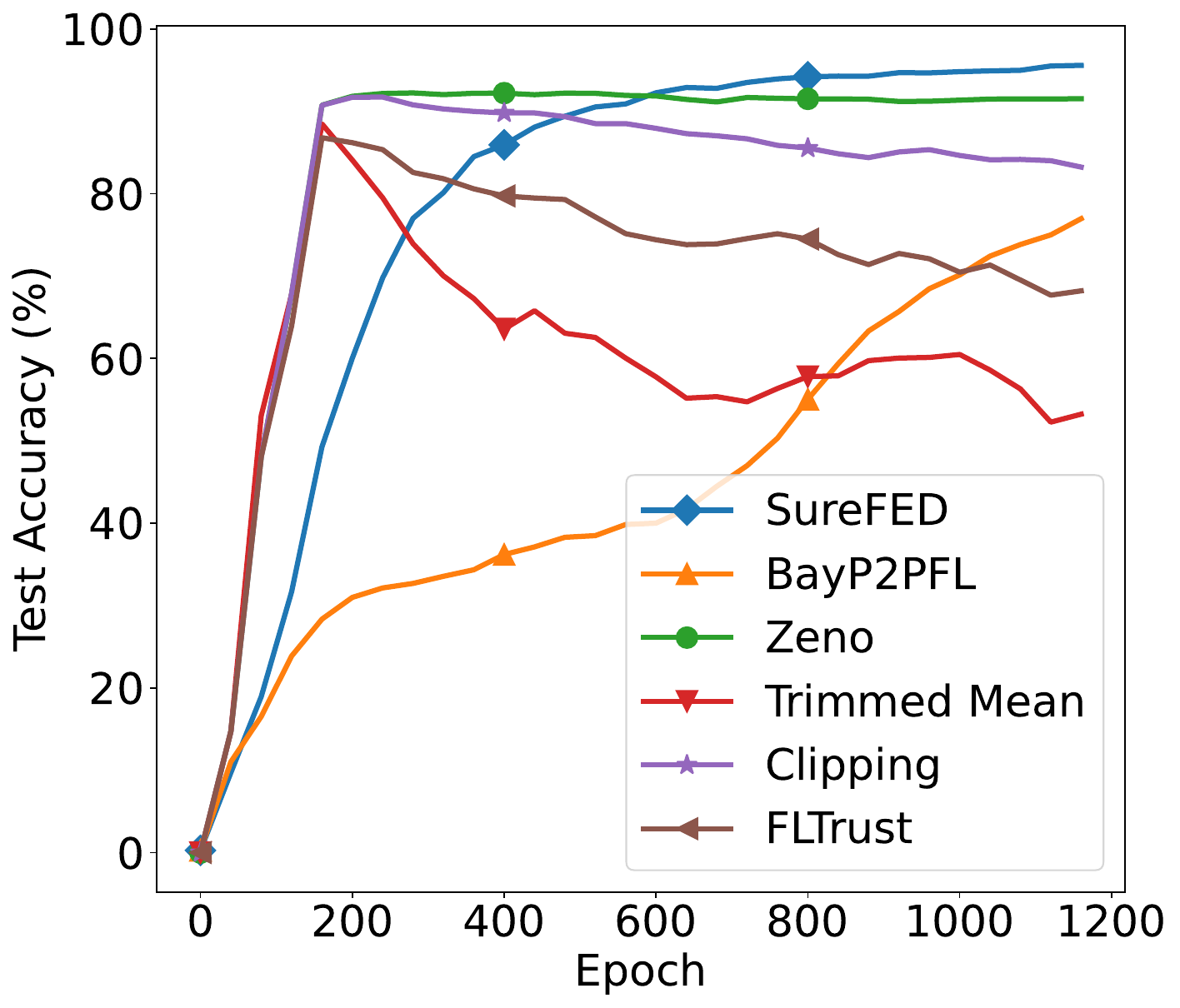}} 
     \subfloat[\scriptsize Bit-Flip Attack]{\label{fig:mnist:bit}\includegraphics[width=0.2\textwidth]{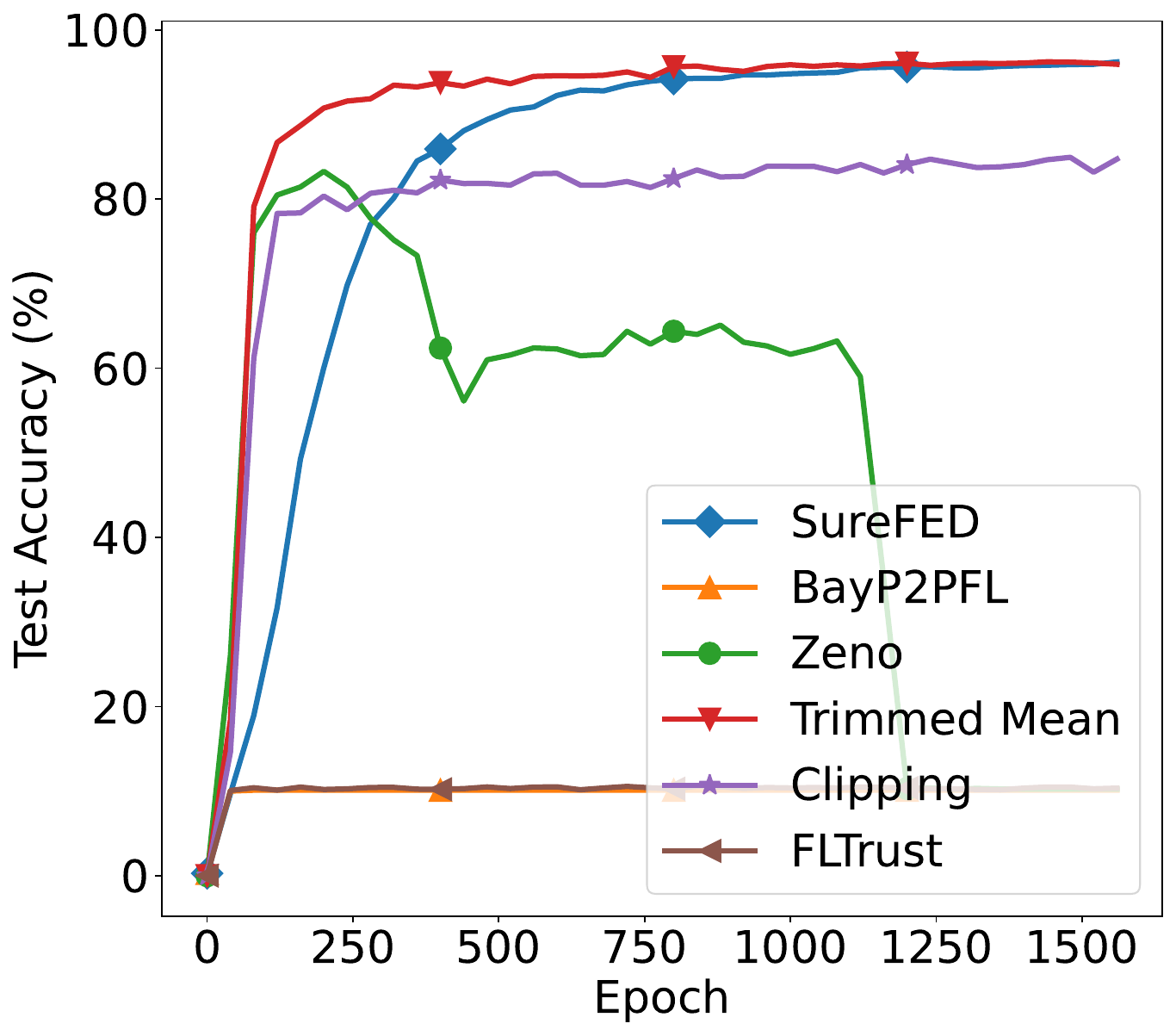}} 
    \subfloat[\scriptsize General Random Attack]{\label{fig:mnist:general}\includegraphics[width=0.2\textwidth]{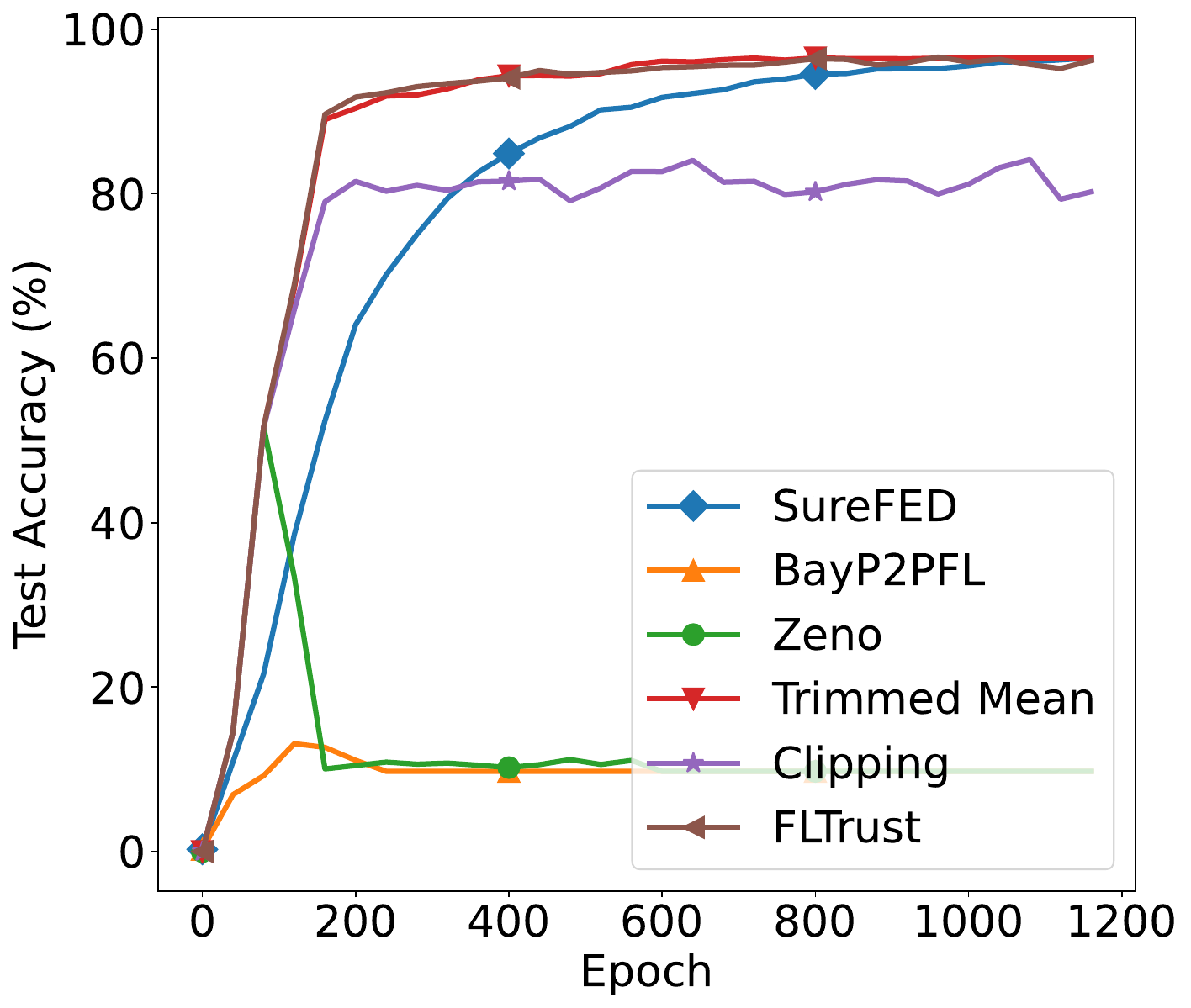}}
      \subfloat[\scriptsize Gaussian Attack]{\label{fig:mnist:gaussian}\includegraphics[width=0.2\textwidth]{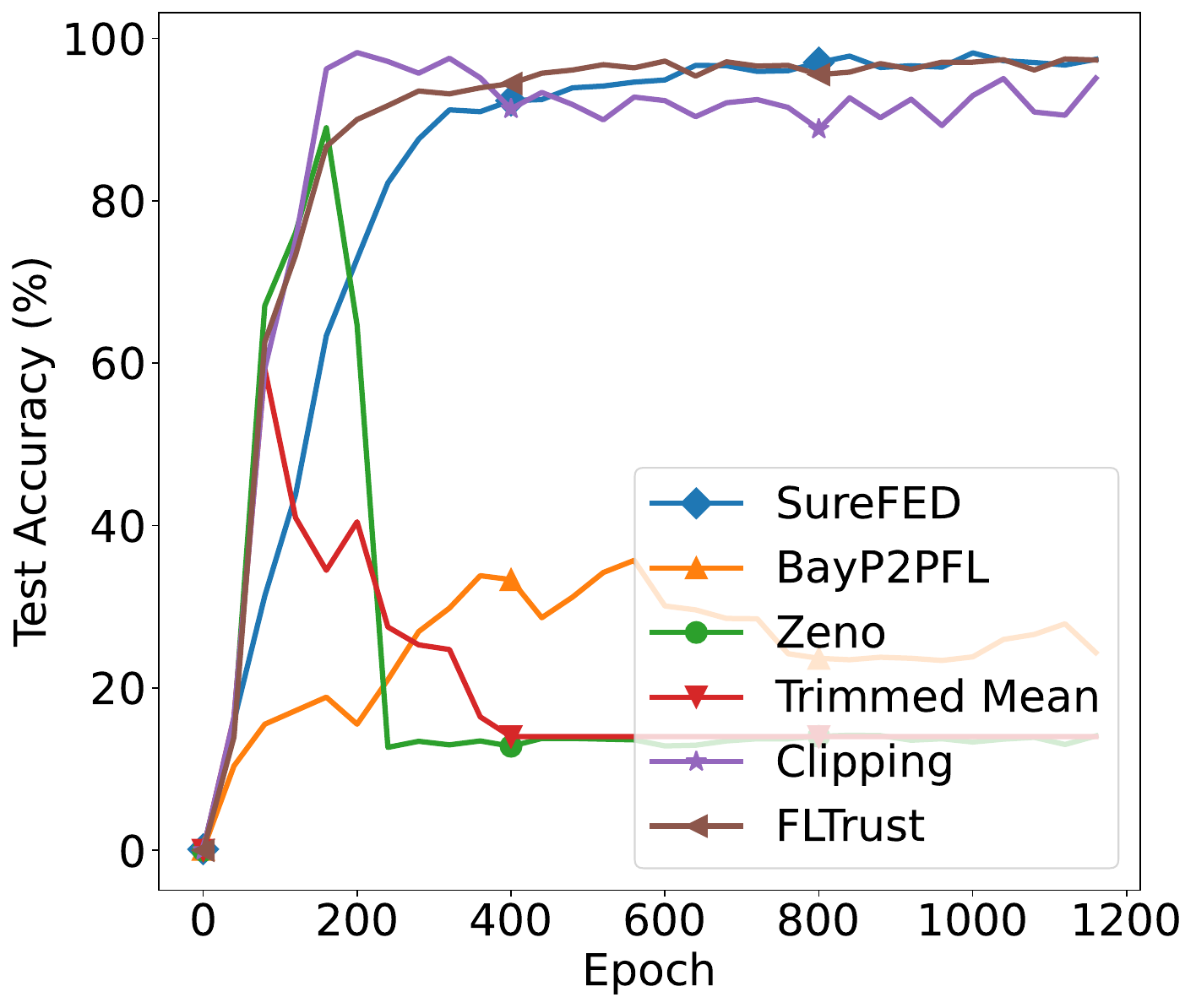}}
  \end{center}
  \caption{Accuracy plot of \sys{} compared with \sysbase{}, Zeno, Trimmed Mean, Clipping, and FLTrust defense methods under different data and model poisoning attacks evaluated on MNIST dataset.  }
  \label{fig:mnist}
\end{figure*}

\begin{figure*}[!ht]
  \begin{center}
    \subfloat[\scriptsize Label-Flipping Attack ]{\label{fig:cifar10:label}\includegraphics[width=0.2\textwidth]{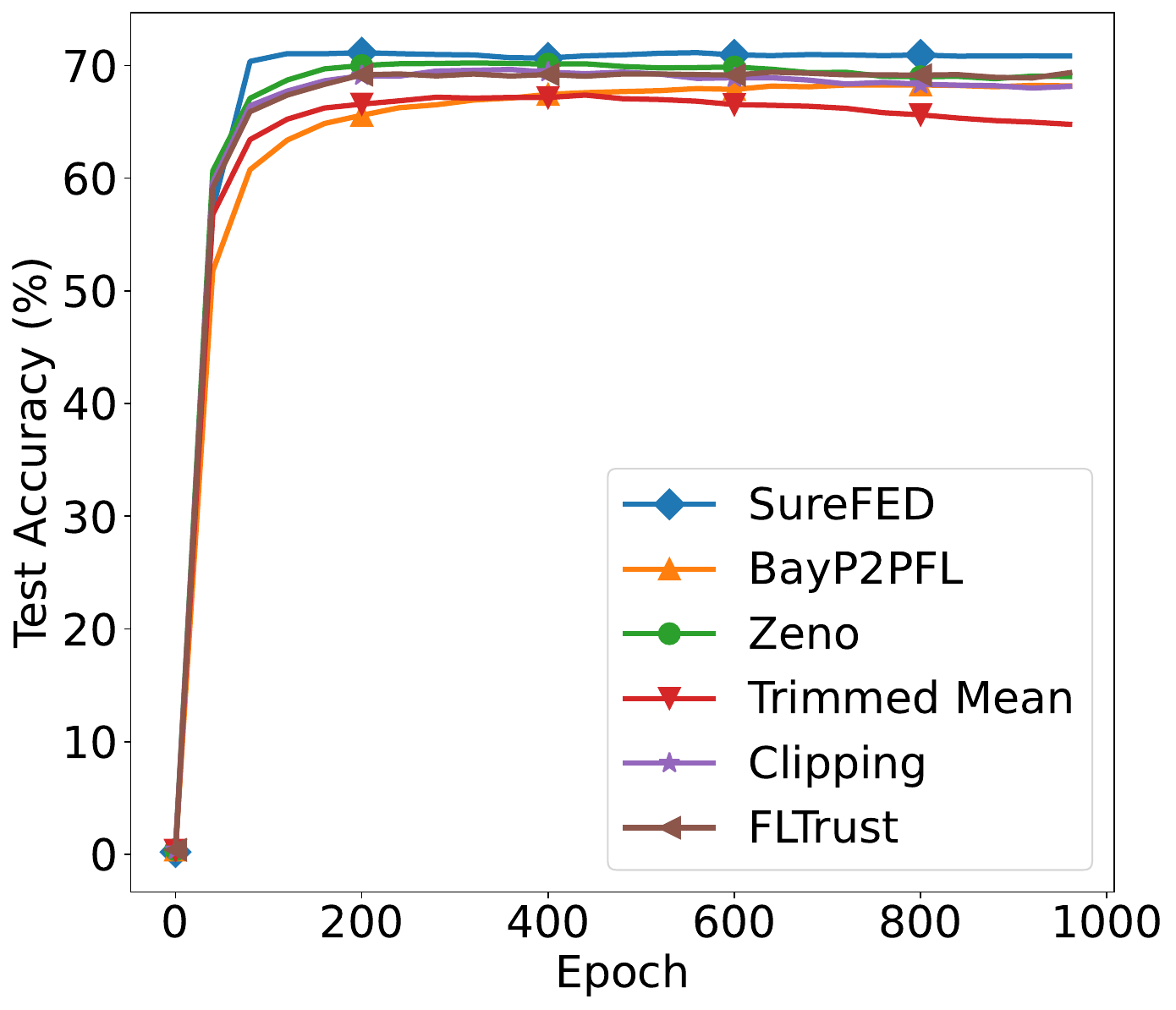}}
     \subfloat[\scriptsize A Little is Enough Attack ]{\label{fig:cifar10:little}\includegraphics[width=0.2\textwidth]{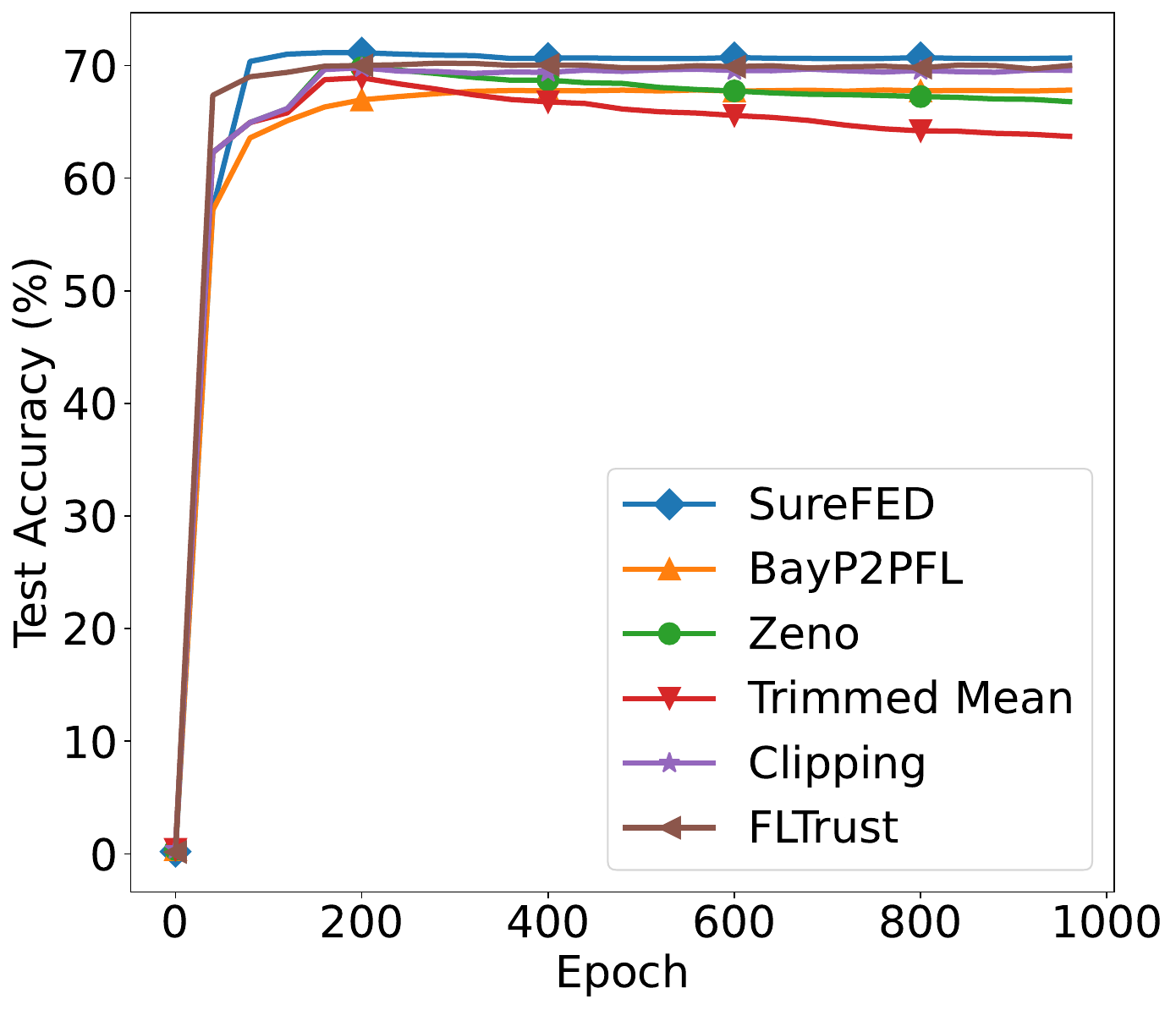}} 
     \subfloat[\scriptsize Bit-Flip Attack]{\label{fig:cifar10:bit}\includegraphics[width=0.2\textwidth]{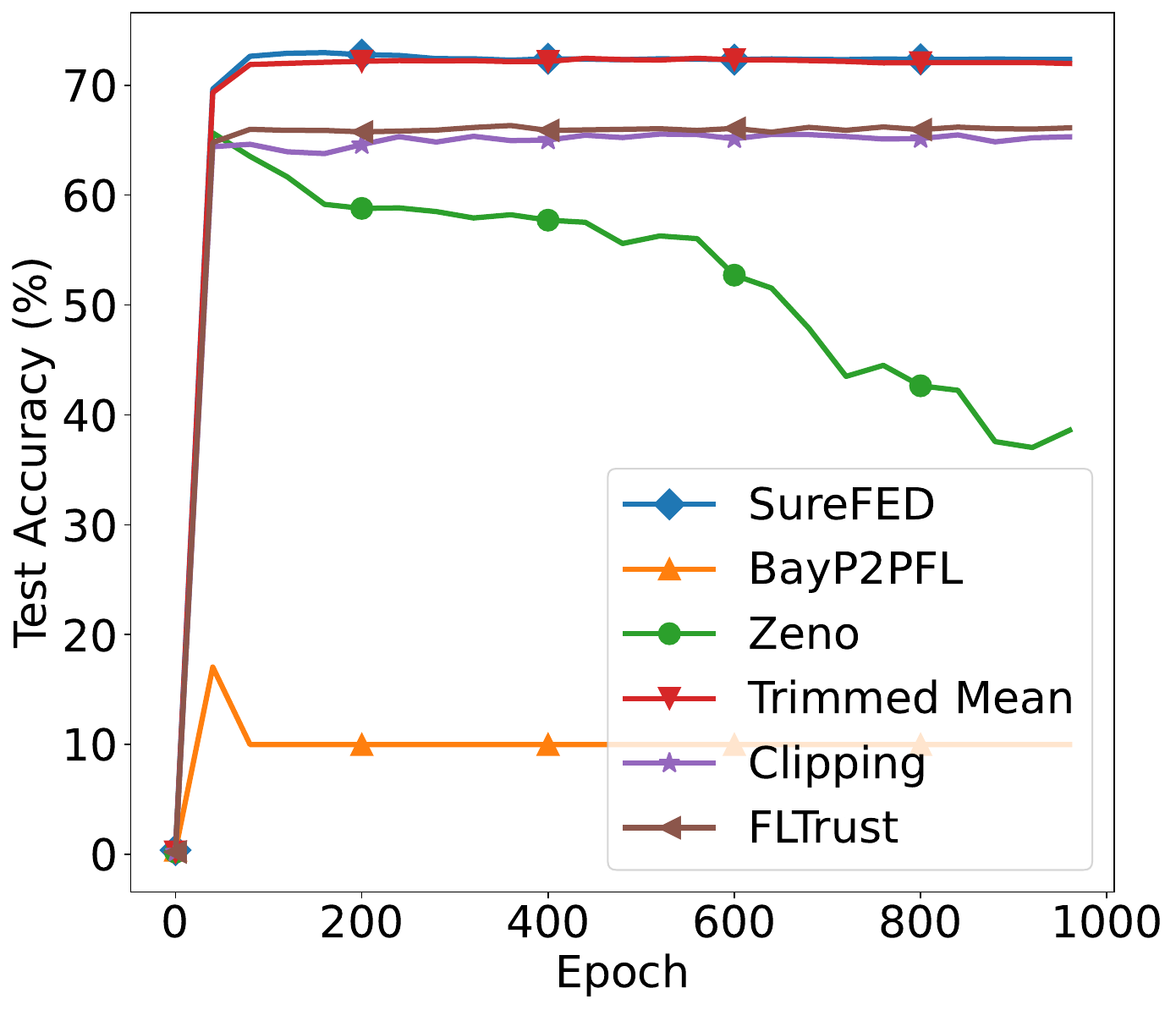}} 
    \subfloat[\scriptsize General Random Attack]{\label{fig:cifar10:general}\includegraphics[width=0.2\textwidth]{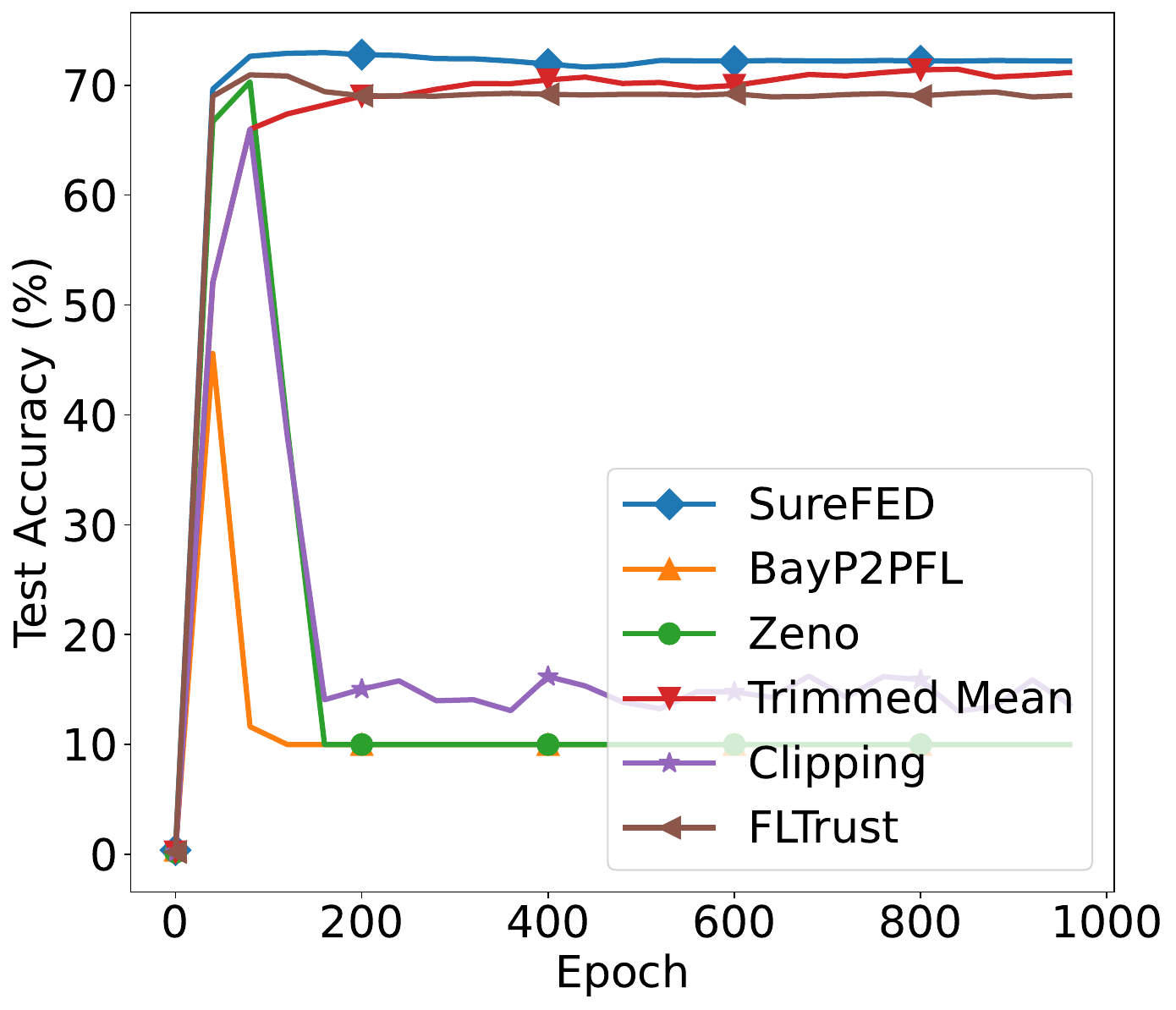}}
     \subfloat[\scriptsize Gaussian Attack]{\label{fig:cifar10:gaussian}\includegraphics[width=0.2\textwidth]{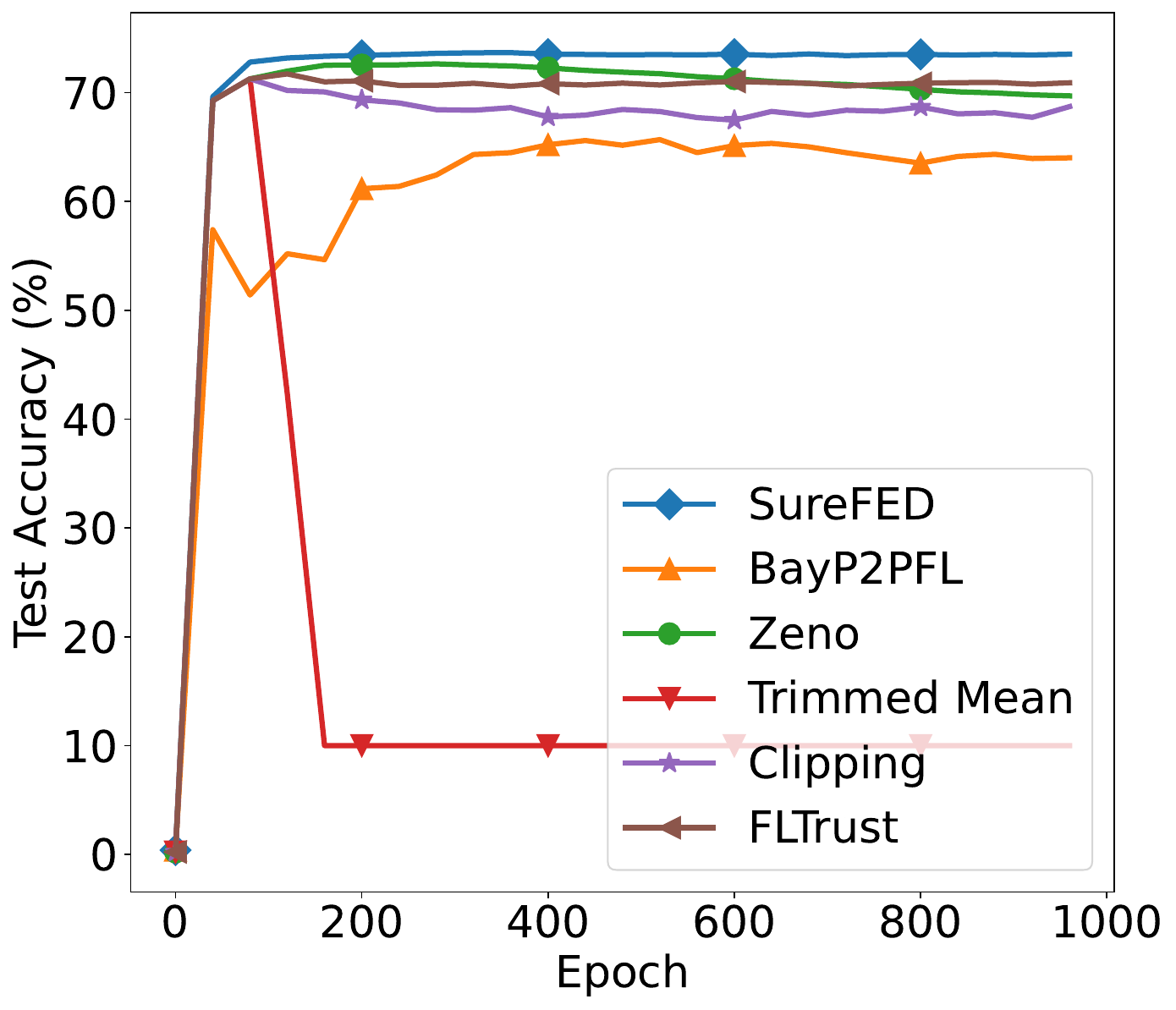}}
  \end{center}
  \caption{Accuracy plot of \sys{} compared with \sysbase{}, Zeno, Trimmed Mean, Clipping, and FLTrust defense methods under different data and model poisoning attacks evaluated on CIFAR10 dataset. }
  \label{fig:cifar10}
\end{figure*}

\begin{table}[!ht]
  \centering
  \small\addtolength{\tabcolsep}{-2pt}
  \begin{tabular}{|c|c|c|c|c|c|}
    \hline
    \backslashbox{{Algorithm}}{{Attack}}  & Label-Flipping & A Little is Enough & Bit-Flip & General Random & Gaussian\\
    \hline
    \textbf{\sys{}}  & \textbf{96\%} &  \textbf{96\%} & \textbf{96\%} & \textbf{96\%} &  \textbf{96\%} \\  \hline
    BayP2PFL & 89\%  & 78\% & 10\%  & 10\%& 22\% \\ \hline
    Zeno & 82\% & 90\% & 35\% & 10\% & 10\% \\ \hline
    Trimmed Mean & 89\% & 56\% & \textbf{96}\% & \textbf{96}\% & 10\% \\  \hline
    Clipping & 82\%  & 82\% & 80\%& 80\% & 94\% \\ \hline
     FLTrust & 10\% & 68\% & 10\% &96 \% & 96\% \\ \hline
     Clean & 96\% & 96\% & 96\% & 96\% & 96\%\\\hline
  \end{tabular}
  \caption{Final Model Accuracy for \sys{} and the other baselines under different attacks  for MNIST dataset. }
  \label{tab:mnist:table}
\end{table}

\begin{table}[!ht]
  \centering
   \small\addtolength{\tabcolsep}{-2pt}
  \begin{tabular}{|c|c|c|c|c|c|}
    \hline
    \backslashbox{Algorithm}{Attack}  & Label-Flipping & A Little is Enough & Bit-Flip & General Random & Gaussian\\
    \hline
    \textbf{\sys{} } & \textbf{73\%}  & \textbf{73\%} & \textbf{73\%} & \textbf{73\%} & \textbf{73\%} \\  \hline
    BayP2PFL & 63\% & 4\% & 2\% & 2\% & 4\% \\ \hline
    Zeno & 51\% & 2\% & 29\% & 2\% &  51\% \\ \hline
    Trimmed Mean & 51\% & 16\% & 70\% & 73\% & 2\% \\  \hline
    Clipping & 67\% & 22\% & 53\% & 34\% & 57\% \\ \hline
      FLTrust & 61\% & 59\% & 54\% & 68\% & 57\% \\ \hline
      Clean & 73\% & 73\% & 73\% & 73\% & 73\% \\\hline
  \end{tabular}
  \caption{Final Model Accuracy for \sys{} and the other baselines under different attacks for FEMNIST dataset.} 
  \label{tab:femnist:table}
\end{table}

\begin{table}[!ht]
  \centering
  \small\addtolength{\tabcolsep}{-2pt}
  \begin{tabular}{|c|c|c|c|c|c|}
    \hline
    \backslashbox{Algorithm}{Attack}  & Label-Flipping & A Little is Enough & Bit-Flip & General Random & Gaussian\\
    \hline
   \textbf{ \sys{}}  & \textbf{71\%}  & \textbf{71\%} & \textbf{71\%} & \textbf{71\%} & \textbf{71\%} \\  \hline
    BayP2PFL & 68\% & 67\% & 10\% & 10\% & 62\% \\ \hline
    Zeno & 53\% & 67\% & 39\% &  10\% & 69\% \\ \hline
    Trimmed Mean & 55\% & 61\% & 71\% & 70\% & 10\% \\  \hline
    Clipping & 68\% & 70\% & 64\% & 14\% & 67\% \\ \hline
    FLTrust & 62\% & 70\% & 64\% &68 \% & 69\% \\ \hline
    Clean & 71\% & 71\% & 71\% & 71\% & 71\%\\\hline
  \end{tabular}
  \caption{Final Model Accuracy for \sys{} and the other baselines under different attacks for CIFAR10 dataset.} 
  \label{tab:cifar10:table}
\end{table}


\begin{table*}[!ht]
\begin{subtable}{0.5\textwidth}
  \centering
  \small\addtolength{\tabcolsep}{-2pt}
  \begin{tabular}{|c|c|c|}
    \hline
    \small{Algorithm} & \small{Main Task Accuracy} & \small{Backdoor Accuracy} \\
    \hline
    \small{\textbf{\sys{} }} &  96\% & \textbf{24\%}  \\  \hline
    \small{BayP2PFL} & 96\% & 100\%\\ \hline
   \small{Zeno} &  96\%  & 100\% \\ \hline
    \small{Trimmed Mean} & 96\%  & 100\%\\  \hline
    \small{Clipping} & 96\%  & 100\%\\ \hline
    \small{FLTrust} & 96\%  & 100\%\\  \hline
  \end{tabular}
  \caption{MNIST Dataset}
  \end{subtable}
  \begin{subtable}{0.5\textwidth}
  \small\addtolength{\tabcolsep}{-2pt}
   \begin{tabular}{|c|c|c|}
    \hline
    \small{Algorithm} & \small{Main Task Accuracy} & \small{Backdoor Accuracy}   \\
    \hline
    \small{\textbf{\sys{}}}  &  73\% &  \textbf{23\%}  \\  \hline
    \small{BayP2PFL} & 73\%  & 79\% \\ \hline
    \small{Zeno} & 73\%  &  82\% \\ \hline
    \small{Trimmed Mean} & 60\%  & 64\% \\  \hline
    \small{Clipping} & 73\%  &  79\% \\ \hline
    \small{FLTrust} & 73\%  &  79\% \\ \hline
  \end{tabular}
  \caption{FEMNIST}
  \end{subtable}
  \caption{Main Task and Backdoor Accuracy of \sys{} and the other baselines under Trojan  attack.  \textbf{High Main Task Accuracy} indicates the success of  the Trojan attack in maintaining its stealthiness (main task accuracy of all baselines should be high). \textbf{Low Backdoor Accuracy} indicates the success of the algorithm in defending against the Trojan attack. }
  \label{tab:trojan}
\end{table*}

\input{3_BCP2PFL}

\subsection{Proofs} \label{proofs}
\begin{proof}[Proof of Theorem \ref{thm:learn_bcp2pfl}]
In order to prove the theorem, we show that the social estimations of agents will converge to $\theta$ and the  social covariance matrices converge to 0. We first assume that the trust weights are fixed throughout the algorithm (similar to \sysbase{} \cite{wang2022peer}) and we will later incorporate the bounded confidence trust weights used in \sys{}.
  \begin{lemma}\label{lm:sigma_p2pfl}
        The covariance matrix of the social beliefs $\bar{\Sigma}^i_t$  decrease with rate proportional to $\frac{1}{t}$, where $t$ is the number of local data observations.
    \end{lemma}
    \begin{proof}
    According to equation \eqref{eq:SoBeUp_alg}, and 
    using Sherman–Morrison formula \cite{hager1989updating}, we can write
     \begin{align*}
    	(\bar{\Sigma}^{i}_t)^{-1}&=(\bar{\Sigma}^i_{t-1}-\frac{\bar{\Sigma}^i_{t-1}x^i_t{x^i_t}^T\bar{\Sigma}^i_{t-1}}{{x^i_t}^T\bar{\Sigma}^i_{t-1}x^i_t+\Sigma^i})^{-1}\\&=(\bar{\Sigma}^{i}_{t-1})^{-1}+\frac{(\bar{\Sigma}^{i}_{t-1})^{-1}\frac{\bar{\Sigma}^i_{t-1}x^i_t{x^i_t}^T\bar{\Sigma}^i_{t-1}}{{x^i_t}^T\bar{\Sigma}^i_{t-1}x^i_t+\Sigma^i}(\bar{\Sigma}^{i}_{t-1})^{-1}}{1-\frac{{x^i_t}^T\bar{\Sigma}^i_{t-1}(\bar{\Sigma}^{i}_{t-1})^{-1}\bar{\Sigma}^i_{t-1}x^i_t}{{x^i_t}^T\bar{\Sigma}^i_{t-1}x^i_t+\Sigma^i}}\\
    	&=(\bar{\Sigma}^{i}_{t-1})^{-1}+\frac{x^i_t{x^i_t}^T}{\Sigma^i} \numberthis \label{eq:Sher_Mor}
    \end{align*}
Notice that the above equation indicates that if  the covariance matrix $\bar{\Sigma}^{i}_{t-1}$ is invertible, then $\bar{\Sigma}^{i}_t$ is also invertible. Since we start with an invertible covariance matrix, all of the covariance matrices are invertible. 
    We denote $Z^i_t=(\bar{\Sigma}^{i}_{t})^{-1}$. Then using equation \eqref{eq:SoBeAgg_alg} and \eqref{eq:Sher_Mor}, we have
    \begin{align}
        Z^i_t=\sum_{j \in \mathcal{N}_t(i)}T^{ij}_t(Z^j_{t-1}+\frac{x^j_t{x^j_t}^T}{\Sigma^j})
        \label{CoIn}
    \end{align}
    If we denote $T_t=(T^{ij}_t)_{i \in \mathcal{N},j\in \mathcal{N}_t(i)}$ (the $ij_{th}$ element of $T_t$ is $T^{ij}_t$ if $j \in \mathcal{N}_t(i)$),  $(Z_t)_{sl}=((Z^i_t)_{sl})_{i \in \mathcal{N}}$, and $(X_t)_{sl}=(\frac{(x^i_t{x^i_t}^T)_{sl}}{\Sigma^i})_{i \in \mathcal{N}}$, we can write
    \begin{align}
       (Z_t)_{sl}=T_t(Z_{t-1})_{sl}+T_t(X_t)_{sl}
    \end{align}
    Therefore, we can write
     \begin{align}
       (Z_t)_{sl}=T_{t:t-\tau(t)}(Z_{\tau(t)-1})_{sl}+\sum_{\tau=t-\tau(t)}^tT_{t:\tau}(X_{\tau})_{sl}
    \end{align}
    where $\tau(t)$ is chosen such that $T_{t:t-\tau(t)}>0$ (all elements are positive). We know such $\tau(t)$ exists due to Assumption \ref{ass:rel_con}. Since the collective dataset of all agents is sufficient to learn the true model parameter, for each $s$, there exists at least one agent $j$ with $(x^j_{\tau}{x^j_{\tau}}^T)_{ss}>0$ with probability 1. Therefore,  $\sum_{\tau=t-\tau(t)}^tT_{t:\tau}(X_{\tau})_{ss}>0$. Hence, on average, at each time step, $\frac{1}{\tau(t)}\sum_{\tau=t-\tau(t)}^tT_{t:\tau}(X_{\tau})_{ss}>0$ is added to $(Z_t)_{ss}$. Therefore,  $tr(Z_t)$ is increasing with rate proportional to $t$. Since $Z_t=(\bar{\Sigma}^{i}_{t})^{-1}$ and $\bar{\Sigma}^{i}_{t}$ is a positive semi-definite matrix, we know that $tr(Z_t)=tr(\bar{\Sigma}^{i}_{t})^{-1}$. Therefore, $tr(\bar{\Sigma}^{i}_{t})$ decreases with rate proportional to $\frac{1}{t}$. 
    Hence, $\bar{\Sigma}^{i}_{t}$ is decreasing with rate proportional to $\frac{1}{t}$.
    
    \end{proof}
    \begin{lemma}\label{lm:estimation_p2pfl}
    The social estimations, $\bar{\theta}^i_t$, converge to $\theta$ for all $i \in \mathcal{N}$.
    \end{lemma}
    \begin{proof}
    According to equation  \eqref{eq:SoBeUp_alg}, we have the following.
     \begin{align*}
      \bar{\theta}^{i}_{1}&=\frac{\Sigma^i_0x^i_1}{{x^i_1}^T\Sigma^i_0x^i_1+\Sigma^i}({x^i_t}^T\theta+\eta^i_t)\\&=\frac{\Sigma^i_0x^i_1{x^i_1}^T}{{x^i_1}^T\Sigma^i_0x^i_1+\Sigma^i}\theta+\frac{\Sigma^i_0x^i_1}{{x^i_1}^T\Sigma^i_0x^i_1+\Sigma^i}\eta^i_t \numberthis
    \end{align*}
    We denote 
\begin{align}
     G^i_t&=I-\frac{\bar{\Sigma}^i_{t-1}x^i_t{x^i_t}^T}{{x^i_t}^T\bar{\Sigma}^i_{t-1}x^i_t+\Sigma^i} \label{eq:G}\\ f^i_t&=\frac{\bar{\Sigma}^i_{t-1}x^i_t}{{x^i_t}^T\bar{\Sigma}^i_{t-1}x^i_t+\Sigma^i}\label{eq:f}
\end{align}
    Based on equation   \eqref{eq:SoBeAgg_alg}, we can write
    \begin{align}
       \bar{\theta}^i_1=\bar{\Sigma}^{i'}_1(\sum_{j \in \mathcal{N}_1(i)} T^{ij}_1(\bar{\Sigma}^j_1)^{-1}((I-G^j_1)\theta+f^j_1\eta^j_1) \label{eq:theta1}
    \end{align}
    Similarly, we can write
     \begin{align}
       \bar{\theta}_1=D(\bar{\Sigma}'_1)\tilde{D}(T_1)D((\bar{\Sigma}_1)^{-1})((I-D(G_1))\tilde{I}\theta+D(f_1)  \eta_1)
    \end{align}
    where we define 
        $\tilde{I}=\left[\begin{array}{c}
             I_K  \\
             I_K\\
             \vdots
        \end{array}\right]
      $,
    and 
    \begin{footnotesize}
     \begin{align}
        \tilde{D}(T_1)=\left[\begin{array}{cccccccc}
             T^{11}_1 & 0 & \cdots   & T^{12}_1 & 0 & \cdots  & 0 & \cdots\\
              0 & T^{11}_1  & \cdots   & 0 & T^{12}_1  & \cdots  & 0 & \cdots\\
              \vdots & & &&&&&\\
               T^{21}_1 & 0 & \cdots  & T^{22}_1 & 0 & \cdots  & 0 & \cdots\\
              0 & T^{21}_1  & \cdots   & 0 & T^{22}_1  & \cdots  & 0 & \cdots\\
             \vdots & & & \cdots &&&&
        \end{array}\right]
    \end{align}
    \end{footnotesize}
    We define
    \begin{align}
      \bar{T}_t=D(\bar{\Sigma}'_t)\tilde{D}(T_t)D((\bar{\Sigma}_t)^{-1})
      \label{eq:wbar}
    \end{align}
     Note that we have $\bar{T}_t\tilde{I}=\tilde{I}$ and therefore,  $\bar{T}_t$ is a row stochastic matrix.
    We can write
     \begin{align}
       \bar{\theta}_1=\bar{T}_1((I-D(G_1))\tilde{I}\theta+D(f_1)\eta_1)
    \end{align}
    Similarly, we have
    \begin{align*}
       \bar{\theta}_2=&\bar{T}_2(D(G_2) \bar{\theta}_1+(I-D(G_2))\tilde{I}\theta+D(f_2)\eta_2)
       \\
       =&\bar{T}_2(D(G_2)\bar{T}_1((I-D(G_1))\tilde{I}\theta+D(f_1)\eta_1) +(I-D(G_2))\tilde{I}\theta+D(f_2)\eta_2)\\
       =&\bar{T}_2\tilde{I}\theta-\bar{T}_2D(G_2))\tilde{I}\theta+\bar{T}_2D(G_2)\bar{T}_1\tilde{I}\theta -\bar{T}_2D(G_2)\bar{T}_1D(G_1)\tilde{I}\theta  +\bar{T}_2D(f_2)\eta_2 +\bar{T}_2D(G_2)\bar{T}_1D(f_1)\eta_1\\
       =& \tilde{I}\theta-\bar{T}_2D(G_2)\bar{T}_1D(G_1)\tilde{I}\theta+\bar{W}_2D(f_2)\eta_2+\bar{T}_2D(G_2)\bar{T}_1D(f_1)\eta_1 \numberthis
    \end{align*}
    where the last equality is due to the fact that $\bar{T}_t \tilde{I}= \tilde{I}$. By generalizing the above equations to time $t$, we can write
    \begin{align}
       \bar{\theta}_t&= \tilde{I}\theta-\bar{T}_tD(G_t)\cdots \bar{T}_1D(G_1)\tilde{I}\theta+\bar{\eta}_t
    \end{align}
     where  $\bar{\eta}_t$ is the cumulative noise terms up to time $t$,
     \begin{align}
     \bar{\eta}_t=&\bar{T}_tD(f_t)\eta_t+\bar{T}_tD(G_t)\bar{T}_{t-1}D(f_{t-1})\eta_{t-1}+ \cdots+\bar{T}_tD(G_t)\cdots \bar{T}_1D(f_1)\eta_1
     \end{align}
     According to the weak law of large numbers, we have $\bar{\eta}_t \xrightarrow{p} 0$. Furthermore, we can show that all eigenvalues of the $G^i_t$ matrices are less than one.  We had
     \begin{align*}
         G^i_t&=I-\frac{\bar{\Sigma}^i_{t-1}x^i_1{x^i_1}^T}{{x^i_1}^T\bar{\Sigma}^i_{t-1}x^i_1+\Sigma^i}=I-\frac{\bar{\Sigma}^i_{t-1}x^i_1{x^i_1}^T}{tr(\bar{\Sigma}^i_{t-1}x^i_1{x^i_1}^T)+\Sigma^i} \numberthis
     \end{align*}
     Since $tr(\bar{\Sigma}^i_{t-1}x^i_1{x^i_1}^T)=\sum_{s=1}^K\lambda^s$, where $\lambda^s$ is the $s_{th}$ eigenvalue of $\bar{\Sigma}^i_{t-1}x^i_1{x^i_1}^T$, all eigenvalues of $ G^i_t$ are less than or equal to one. Notice that $\bar{\Sigma}^i_t$ is a covariance matrix and therefore, is positive semi-definite. But since our covariance matrices are invertible, they  are positive definite. Hence,  if all elements of $x^i_t$ are positive, then all eigenvalues of $G^i_t$ would be less than one.
     The eigenvalues of $G^i_t$ that are one are due to some elements of $x^i_t$ being zero, thus making some rows of $\bar{\Sigma}^i_{t-1}x^i_1{x^i_1}^T$ to be all zeros. Therefore, contraction would not happen on those rows.
     However,  for each element of $\theta$, e.g., $(\theta)_k$, there is at least one agent with non-zero  $(x^i_t)_k>0$.  Based on assumption \ref{ass:rel_con}, for some set $\{s_1, \cdots, s_t\}$, we have   $\prod_{s=s_{t-1}}^{s_t} \bar{T}_sD(G_s)$ to have eigenvalues that are less than one (each block of $\bar{T}_t$ is a row stochastic matrix and therefore, all its eigenvalues are less than or equal to one, and we also assume that if $T^{ij}_t>0,$ then $T^{ij}_t>\delta$ for some $\delta>0$) and thus, $\prod_{s=s_{t-1}}^{s_t} \bar{T}_sD(G_s)$ is a contraction.
     Therefore, we have $\lim_{t\rightarrow \infty } ||\bar{T}_tD(G_t)\cdots \bar{T}_1D(G_1)\tilde{I}\theta||=0$. Consequently, we have $\lim_{t\rightarrow \infty } \bar{\theta}_t=\tilde{I}\theta$.

    \end{proof}
    Using lemmas \ref{lm:sigma_p2pfl} and \ref{lm:estimation_p2pfl}, the estimation of clients converges to the true model parameter with their uncertainty (variance) converging to zero.  Hence, clients can learn the true model parameter with sufficient data observations.

The difference that \sys{} has with the setting of the above proof above is  that $T_t$  changes by time according to our robust aggregation rule. Therefore, in order to prove the theorem, it suffices to show that $T_t$  will satisfy the relaxed connectivity constraint.  Since $A_t$ satisfies the relaxed connectivity constraint, we need to show that if there is an edge between users $i$ and $j$, then $T^{ij}_t>\delta$ with high probability for some $\delta>0$. 

Since the beliefs are Gaussian, for agent $i$ we have $|(\hat{\theta}^i_t)_k-(\theta)_k|<2\sqrt{(\hat{\Sigma}^i_t)_{k,k}}$ for all $k \in \mathcal{K}^i$ with high probability ($\approx 0.97$). The same is true for the social belief of agent $j$. Therefore, we have $|(\hat{\theta}^i_t)_k-(\bar{\theta}^j_t)_k|<2\sqrt{(\hat{\Sigma}^i_t)_{k,k}}+2\sqrt{(\bar{\Sigma}^j_t)_{k,k}}$ with high probability ($\approx 0.999$). Hence, if $(\hat{\Sigma}^i_t)_{k,k}\simeq (\bar{\Sigma}^j_t)_{k,k}$, and for $\kappa \approx 4$, agent  $i$ will aggregate the updates of agent $j$ with high probability. Notice that there is a possibility for agent $i$ to not aggregate the updates of agent $j$ when $(\bar{\Sigma}^j_t)_{k,k}\gg (\hat{\Sigma}^i_t)_{k,k}$. This indicates that if the model of agent $j$ is too bad, agent $i$ will not aggregate it. 
\end{proof}



\begin{proof}[Proof of Theorem \ref{thm:bcp2pfl_robust}]
Before proving this theorem, we need to state and prove the following theorem. 
\begin{theorem}
If the trust weight matrix $T_t$ is fixed through time (similar to \sysbase{}), and if nodes $\mathcal{N}^c$ are compromised by Label-Flipping attack and the communication graph satisfies the relaxed connectivity constraint, the estimation of all users converge to $\theta^*+cb$ for some vector $c$ with positive elements. 
 \label{thm:p2pfl_dp}
\end{theorem}
    \begin{proof}[Proof of Theorem \ref{thm:p2pfl_dp}]
    We denote the benign nodes (the nodes that are not compromised) by $\mathcal{N}^b$. Also, the benign neighbors of node $i$ at time $t$ are denoted by $\mathcal{N}^b_t(i)$. Similarly, we denote the compromised neighbors of node $i$ at time $t$ by $\mathcal{N}^c_t(i)$. 
    According to equation  \eqref{eq:SoBeAgg_alg} and \eqref{eq:SoBeUp_alg}, and similar to equation \eqref{eq:theta1}, we can write 
   \begin{align}
       \bar{\theta}^i_1
       &= \bar{\Sigma}^{i'}_1(\sum_{j \in \mathcal{N}_1(i)} T^{ij}_1(\bar{\Sigma}^j_1)^{-1}((I-G^j_1)\theta+f^j_1\eta^j_1)) +\bar{\Sigma}^i_1\sum_{j \in \mathcal{N}^c_1(i)} T^{ij}_1(\bar{\Sigma}^j_1)^{-1}f^j_1b 
    \end{align}
    Therefore, we have
     \begin{align}
       \bar{\theta}_1=\bar{T}_1((I-D(G_1))\tilde{I}\theta+D(f_1)\eta_1)+\bar{T}_1D(f_1) \mathbf{1}^cb
    \end{align}
    where $\tilde{I}$ and $\bar{T}_1$  are defined in the proof of Theorem \ref{thm:learn_bcp2pfl}. We also define $(\mathbf{1}^c)_{(j-1)K:jK}=1$ if $j \in \mathcal{N}^c$ and $(\mathbf{1}^c)_{(j-1)K:jK}=0$, otherwise.
    Similarly, we can write
     \begin{align}
       \bar{\theta}_t&= \tilde{I}\theta-\bar{T}_tD(G_t)\cdots \bar{T}_1D(G_1)\tilde{I}\theta+\bar{\eta}_t+\mathbf{c}_t b 
       \label{eq:bartheta_c}
    \end{align}
    where 
    \begin{align}
        \mathbf{c}_t=&\bar{T}_tD(f_t) \mathbf{1}^c+\bar{T}_tD(G_t)\bar{T}_{t-1}D(f_{t-1}) \mathbf{1}^c+ \cdots+\bar{T}_tD(G_t)\cdots \bar{T}_1D(f_1) \mathbf{1}^c 
    \end{align}
    In the proof of Theorem \ref{thm:learn_bcp2pfl} we showed that the sum of the first two terms in equation \eqref{eq:bartheta_c} will converge to $\theta$. In the following, we will show that $\mathbf{c}_t$ will converge to $\tilde{I}\mathbf{c}$, where $\mathbf{c}$ is a vector of size $K$  with positive elements. We can write
    \begin{align}
       \mathbf{c}_t=\bar{T}_tD(f_t) \mathbf{1}^c+\bar{T}_tD(G_t)\mathbf{c}_{t-1}
    \end{align}
    We know that $f^i_t \rightarrow 0$ and $G^i_t \rightarrow I$. Therefore, for large enough $t$,  we have
    \begin{align}
        \mathbf{c}_t\simeq\bar{T}_t\mathbf{c}_{t-1}
        \label{c_t}
    \end{align}
    We show the convergence of $\mathbf{c}_t$ based on the convergence analysis in the opinion dynamics literature \cite{hegselmann2002opinion,krause2000discrete}.  In this literature, there are a network of agents, each of which has an opinion over a specific matter. The opinions of the agent $i$ at time $t$ is denoted by $\mathbf{x}^i_t$, and the vector of agent's opinions is denoted by $\mathbf{x}_t$. 
    The opinion dynamics is usually assumed to be $\mathbf{x}_{t+1}=A_t\mathbf{x}_t$, where $A_t$ is a row stochastic matrix. One can see that $\mathbf{c}_t$ evolves according to a similar model. We  first state the next well known lemma from \cite{seneta2006non} (Theorem 3.1).
    \begin{lemma}
        If $A$ is a row stochastic matrix, then we have
        \begin{align}
            v(A\mathbf{x}) \leq (1-\min_{1 \leq i,j \leq n}\sum_{k=1}^n\min\{a_{ik},a_{jk}\}) v(\mathbf{x}),
        \end{align}
        where $\mathbf{x}\in \mathbb{R}^n$ and
    \begin{align}
        v(\mathbf{x})=\max_{1 \leq i \leq n}\mathbf{x}^i-\min_{1 \leq i \leq n}\mathbf{x}^i=\min_{1 \leq i,j \leq n}(\mathbf{x}^i-\mathbf{x}^j)
    \end{align}
        \label{rowstoch}
    \end{lemma}
    The proof of the above lemma can be found in \cite{seneta2006non}.
    
     We also state the following condition on a row stochastic matrix $A_t$ based on \cite{krause2000discrete}. 
     Suppose there exists numbers $0 \leq \delta_t\leq 1$ such that $\sum_{t=0}^{\infty} \delta_t=\infty$ and $\sum_{k=1}^n \min\{a^{ik}_t,a^{jk}_t\}\geq \delta_t$. We refer to this condition as the  joint connectivity condition. 
     Using this condition and according to \cite{krause2000discrete}, for the model of $\mathbf{x}_{t+1}=A_t\mathbf{x}_{t}$, we have $v(\mathbf{x}_t) \leq e^{-\sum_{s=0}^t \delta_s}v(\mathbf{x}_0)$. Therefore, we have $v(\mathbf{x}_t) \rightarrow 0$ and thus, $\mathbf{x}_t$ converges to a vector with the same elements, which we denote as $\mathbf{x}=x \mathbf{1}$. We note that one can relax the condition on the matrix $A_t$, by defining a set $\{s_1, s_2, \cdots, s_t, \cdots\}$ and $\bar{A}_t=\prod_{l=s_{t}}^{s_{t-1}}A_l$. If the joint connectivity condition holds for $\bar{A}_t$, then the result still holds.  We refer to this relaxed condition as the relaxed joint connectivity condition. 
     
     In order to apply the above argument to prove the convergence of $\mathbf{c}_t$, we  assume that the non-diagonal entries of $\bar{\Sigma}^i_t$ are 0 for all $i \in \mathcal{N}$. That is, we assume that the belief parameters are shared separately for each element of the model parameter. This is what happens in the general variational learning version in Section \ref{sec:SABRE}.   Notice that this modification does not affect the learning analysis of the  algorithm done in Theorem \ref{thm:learn_bcp2pfl}. We denote the vector $\mathbf{c}^k_t=(\mathbf{c}^{(i-1)*K+k}_t)_{i \in \mathcal{N}}$ to be the vector of elements of $\mathbf{c}_t$ corresponding to the $k_{th}$ element of the model parameter, then $\mathbf{c}^k_t$ evolves according to the opinion dynamics model of $\mathbf{c}^k_t=\bar{T}^k_t\mathbf{c}^k_{t-1}$ for all $k \in \mathcal{K}$, where $\bar{T}^k_t$ is the matrix consisting of the elements of $\bar{T}_t$ corresponding to $\mathbf{c}^k_{t}$ and $\mathbf{c}^k_{t-1}$ (the elements that are multiplied to $\mathbf{c}^k_{t-1}$ to generate $\mathbf{c}^k_{t}$  in equation \eqref{c_t}). $\bar{T}^k_t$ is a row stochastic matrix. We further mention that the relaxed connectivity condition in Assumption \ref{ass:rel_con} will imply that $\bar{T}^k_t$ satisfies the relaxed joint connectivity condition described earlier. Therefore, $\mathbf{c}^k_{t}$ will converge to a vector of the same elements, $\mathbf{c}^k \mathbf{1}$. Consequently, $\mathbf{c}_{t}$ converges to a vector $\tilde{I}\mathbf{c}$. 
Note that $c>0$ because  according to equation \eqref{c_t}, $\mathbf{c}_t$ is a summation of all positive terms.  
    \end{proof}

In order to show the robustness of \sys{}, we will show that the social estimations of the benign clients converge to $\theta$ and the social covariance matrices converge to 0. 

In the proof of Theorem \ref{thm:learn_bcp2pfl}, we showed that  the covariance matrices of the beliefs converge to 0. Furthermore, we can see in the proof of Lemma \ref{lm:sigma_p2pfl} that the covariance matrices of the beliefs are independent of the labels $y^i_t$ and therefore, any bias in the labels will not affect the covariance matrices. Therefore, the covariance matrices will converge to 0 whether or not some users are compromised.

   Similar to the proof of Theorem \ref{thm:p2pfl_dp},    this proof is also based on the convergence analysis in the opinion dynamics literature \cite{hegselmann2002opinion,krause2000discrete}.
    In this proof, the opinion of users are their social estimations of $\theta$. That is, we have $x^i_t=\bar{\theta}^i_t$. We will model the data acquisition of agents by adding two benign and malicious nodes to the network that play the role of stubborn opinion leaders \cite{zhao2016bounded} whose opinions are always $\theta$ and $\theta+b$, respectively, for the benign and malicious opinion leaders.  The opinion of these two nodes will not change. The benign nodes directly hear the opinion of benign opinion leader with environment noise (data observation noise) and the compromised nodes, directly hear the opinion of malicious opinion leader. According to this model, the opinion dynamic of agents in the network is given by linear equations \eqref{eq:SoBeUp_alg} and \eqref{eq:SoBeUp_alg}.

    In order to prove that \sys{} is robust, we show that during the dynamics of the estimations of users, we will see an opinion fragmentation (see \cite{hegselmann2002opinion}) between the benign and the compromised agents. If an opinion fragmentation happens for a benign user and a compromised one, the benign agent will not aggregate the belief of the compromised agent.  In order to show that an opinion fragmentation happens, we show that a compromised user will be removed from the confidence set of agent $i$ at some time $t$. 
    
   Similar to Lemma \ref{lm:sigma_p2pfl}, one can show that   the variances of the local belief of user $i$ for the elements $k \in \mathcal{K}^i$, decrease with rate proportional to $\frac{1}{t}$,  where $t$ is the number of local observations.
    Based on the proof of Theorem \ref{thm:learn_bcp2pfl}, one can easily show that the local estimation of the benign user $i$ on $(\theta)_k$  converge to $(\theta)_k$ for $k \in \mathcal{K}^i$,  by setting the trust weights $T^{ij}_t$ to 0 except for $j \neq i$. 
    
  Assume that agent $j$ is compromised and agent $i$ is benign and they can communicate with each other at some time. According to Assumption \ref{ass:Jlearning}, there exists a $k\in \mathcal{K}^i \cap \mathcal{K}^j$. 
   Since the local estimation of agent $i$ on $(\theta)_k$ is converging to $(\theta)_k$ and the social estimation of agent $j$ is converging to $(\theta+\mathbf{c}b)_k$ for some $\mathbf{c}>0$, there is a time $t$ at which we have $|(\hat{\theta}^i_t)_k-(\bar{\theta}^j_t)_k|>\kappa\sqrt{(\hat{\Sigma}^i_t)_{k,k}}$. It can be easily proved by contradiction. That is, assume that it does not happen. Since  $(\hat{\Sigma}^i_t)_{k,k}$ is converging to 0, we must have $(\hat{\theta}^i_t)_k$ and $(\bar{\theta}^j)_k$ converge to the same point, which we know is a contradiction. Next, we will investigate how many observations are needed from agent $i$ to make an opinion fragmentation with a compromised agent $j$. 
  We note that since the beliefs are Gaussian, $|(\hat{\theta}^i_t)_k-(\theta)_k|<2\sqrt{(\hat{\Sigma}^i_t)_{k,k}}$ with high probability ($\approx 0.97$) and the same can be said about other beliefs as well.
Assume we have $\sqrt{(\bar{\Sigma}^j_t)_k}\leq a\sqrt{(\hat{\Sigma}^i_t)_k}$.
   If we have $|(\theta)_k-(\theta+cb)_k|>(2a+2+\kappa)\sqrt{(\hat{\Sigma}^i
   )_k}$, we know with high probability that $|(\hat{\theta}^i_t)_k-(\bar{\theta}^j_t)_k|>\kappa\sqrt{(\hat{\Sigma}^i_t)_{k,k}}$. Note that the multiplier $2a+2+\kappa$ is due to the uncertainty of $(\hat{\theta}^i_t)_k$ around $(\theta)_k$ (which is $2\sqrt{(\hat{\Sigma}^i)_{k,k}}$), the uncertainty of $(\bar{\theta}^i_t)_k$ around $(\theta+\mathbf{c}b)_k$ (which is $2\sqrt{(\bar{\Sigma}^j_t){k,k}}\leq 2\kappa\sqrt{(\hat{\Sigma}^i_t)_{k,k}}$) and the allowed deviation of $(\hat{\theta}^i_t)_k$ from $(\bar{\theta}^j_t)_k$ (which is $\kappa\sqrt{(\hat{\Sigma}^i_t)_{k,k}}$). Therefore, if we have 
        $(\hat{\Sigma}^i_t)_{k,k}< (\frac{\mathbf{c}_kb}{2a+2+\kappa})^2$,
    then we have an opinion fragmentation with high probability. 
   Since $(\hat{\Sigma}^i_t)_{k,k}$ decreases with rate proportional to $\frac{1}{t}$, we should have
	\begin{align}
        t> l((\frac{2a+2+\kappa}{\mathbf{c}_kb})^2)
    \end{align}
	for the opinion fragmentation to happen. The multiplier $l$ is derived from the average decrease in $(\hat{\Sigma}^i_t)_{k,k}$ by each local data observation.

 Notice that the last step  of \sys{} was to make sure that the social model of clients stay close to their local models for those elements that a client has a local model. This step will ensure that even if at the beginning of the algorithm (before the opinion fragmentation with the compromised clients happens), the social models get poisoned, the clients can correct them with their clean local models. Therefore, after the compromised clients are removed from the system, everything will become the same as the learning in a benign  setting and according to Theorem \ref{thm:learn_bcp2pfl}, clients will learn the true model parameter if the conditions of the learning hold. 
 
	In conclusion, we can say that \sys{} algorithm is robust against the considered Label-Flipping data poisoning attacks.

\end{proof}


\begin{proof}[Proof of Theorem \ref{thm:bcp2pfl_robust_gen}]
 Assume that a given attacker $j$ is tampering with the model parameter weights in set $\mathcal{K}^c$ with size $|\mathcal{K}^c|=C|\mathcal{K}|$. Also, assume that for  a given client $i$, we have $|\mathcal{K}^i|=L|\mathcal{K}|$. In order for client $i$ to detect the poisoned client $j$, we need to have $\mathcal{K}^i \cap \mathcal{K}^c \neq \emptyset$. Since the model parameter elements in $\mathcal{K}^c$ are chosen randomly, the probability of the above condition holding is as follows.
\begin{align}
    \mathbb{P}(\mathcal{K}^i \cap \mathcal{K}^c \neq \emptyset)&=1-\frac{{|\mathcal{K}|-|\mathcal{K}^i| \choose |\mathcal{K}^c|}}{{|\mathcal{K}| \choose |\mathcal{K}^c|}}= 1-\frac{{(1-L)|\mathcal{K}| \choose C|\mathcal{K}|}}{{|\mathcal{K}| \choose C|\mathcal{K}|}}
\end{align}
Fig. \ref{fig:prob_v} shows the plot of the above probability w.r.t. the model size for different values of $C$ and $L$, and in Fig. \ref{fig:prob_f}, we see the probability plot for a fixed model size (1e6) and different values of $C$ and $L$.  We see that the probability is almost always equal to 1 for different ranges of $L$ and $C$ and model sizes. 
\begin{figure}[H]
\centering
    \subfloat[Variable Model Size]{   \label{fig:prob_v}  \includegraphics[width=4.5cm]{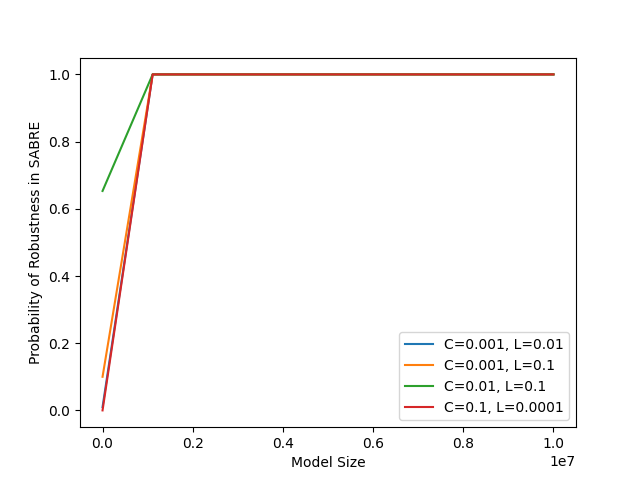}}
    \subfloat[Model Size Fixed at 1e6]{   \label{fig:prob_f}  \includegraphics[width=4cm]{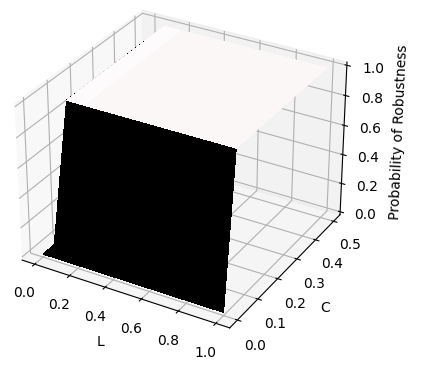}}
    \caption{The probability of \sys{} being robust against General Random attack.}
   \label{fig:prob}
\end{figure}
\end{proof}

%% file: 3_BCP2PFL.tex
\subsection{ \sys{} for Decentralized Linear Regression} 
\label{sec:bc_p2p_fl}
In this section, we describe \sys{} for the case where the models have a single linear layer.  We also assume that the labels are generated according to the linear equation of $y^i_t=\langle\theta^*,x^i_t\rangle+\eta^i_t$, where $\langle\rangle$ denotes the inner product and we assume $\eta^i_t\sim N(0,\Sigma^i)$ is a Gaussian  noise. In this setting, the local belief updates are done according to  Bayes rule and there is no need to fit the posteriors to a Gaussian distribution  due to the posteriors themselves  being Gaussian distributions. 
Therefore, the  belief updates that are done locally are simplified to Kalman filter updates on the parameters of the Gaussian distributions. In particular,  after receiving a data sample $(x^i_t,y^i_t)$, the belief parameters are updated as follows. 
 \begin{subequations}
	\label{eq:PrBeUp_alg}
\begin{align}
\hat{\theta}^{i}_t&=\hat{\theta}^i_{t-1}+\frac{\hat{\Sigma}^i_{t-1}x^i_t}{{x^i_t}^T\hat{\Sigma}^i_tx^i_t+\Sigma^i}(y^i_t-{x^i_t}^T\hat{\theta}^i_{t-1}) \label{eq:PrEsUp_alg}\\
	\hat{\Sigma}^{i}_t&=\hat{\Sigma}^i_{t-1}-\frac{\hat{\Sigma}^i_{t-1}x^i_t{x^i_t}^T\hat{\Sigma}^i_{t-1}}{{x^i_t}^T\hat{\Sigma}^i_{t-1}x^i_t+\Sigma^i}
	\label{eq:PrCoUp_alg}
\end{align}
\end{subequations} 
\begin{subequations}
    \begin{align}
		\bar{\theta}^{i}_t&=\bar{\theta}^i_{t-1}+\frac{\bar{\Sigma}^i_{t-1}x^i_t}{{x^i_t}^T\bar{\Sigma}^i_{t-1}x^i_t+\Sigma^i}(y^i_t-{x^i_t}^T\bar{\theta}^i_{t-1})\\
		\bar{\Sigma}^{i}_t&=\bar{\Sigma}^i_{t-1}-\frac{\bar{\Sigma}^i_{t-1}x^i_t{x^i_t}^T\bar{\Sigma}^i_{t-1}}{{x^i_t}^T\bar{\Sigma}^i_{t-1}x^i_t+\Sigma^i} 
	\end{align}
	\label{eq:SoBeUp_alg}
\end{subequations}